\documentclass[twoside,10pt]{jmlr_homebrew} 

\newcommand{\jmlrBlackBox}{\rule{1.5ex}{1.5ex}}
\providecommand{\BlackBox}{\jmlrBlackBox}

\pdfoutput=1

\title[Online-to-PAC Conversions]{Online-to-PAC Conversions: \\Generalization Bounds via 
Regret Analysis}
\usepackage{times}
\usepackage[utf8]{inputenc} 
\usepackage[T1]{fontenc}    
\usepackage{hyperref}       
\usepackage{url}            
\usepackage{booktabs}       
\usepackage{amsfonts}       
\usepackage{nicefrac}       
\usepackage{microtype}      
\usepackage{xspace}
\usepackage{amsmath}

\usepackage{mathtools}  
\usepackage{amsmath}
\usepackage{amssymb}
\usepackage{tabulary}
\usepackage{booktabs}

\usepackage{algorithm}
\usepackage{algpseudocode}
\usepackage{algpascal}
\usepackage{comment}
\usepackage{selectp}
\usepackage{selectp}

\newcommand{\bgen}{\overline{\gen}(W_n,S_n)}
\newcommand{\sm}{\gamma}
\newcommand{\dd}{\mathrm{d}}
\newcommand{\intW}{\int_{\Ww}}

\newcommand{\F}{\mathcal{F}}

\newcommand{\NN}{\mathbb{N}}
\newcommand{\real}{\mathbb{R}}
\newcommand{\C}{\mathcal{C}}

\newcommand{\Pw}{\mathcal{P}}

\newcommand{\Ww}{\mathcal{W}}

\newcommand{\DD}[2]{\mathcal{D}\pa{#1\middle\|#2}}
\newcommand{\DDh}[2]{\mathcal{B}_{\bH}\pa{#1\middle\|#2}}

\newcommand{\DDsigma}[2]{\mathcal{D}_{\sm}\pa{#1\middle\|#2}}
\newcommand{\DDKL}[2]{\mathcal{D}_{\mathrm{KL}}\pa{#1\middle\|#2}}
\newcommand{\DDKLb}[2]{\mathcal{D}_{\mathrm{KL}}\bigl(#1\bigm\|#2\bigr)}

\newcommand{\DDPhi}[2]{\mathcal{B}_\Phi\pa{#1\middle\|#2}}
\newcommand{\DDhstar}[2]{\mathcal{B}_{h^*}\pa{#1\middle\|#2}}

\newcommand{\DDchi}[2]{\mathcal{D}_{\chi^2}\pa{#1\middle\|#2}}

\newcommand{\DDU}[2]{\mathcal{B}_U\pa{#1\middle\|#2}}

\newcommand{\PP}[1]{\mathbb{P}\left[#1\right]}

\newcommand{\EE}[1]{\mathbb{E}\left[#1\right]}

\newcommand{\EXP}{\mathbb{E}}
\newcommand{\EEs}[2]{\mathbb{E}_{#2}\left[#1\right]}

\newcommand{\EEt}[1]{\mathbb{E}_t\left[#1\right]}

\newcommand{\EEcc}[2]{\mathbb{E}\left[\left.#1\right|#2\right]}

\def\argmin{\mathop{\mbox{ arg\,min}}}

\newcommand{\ra}{\rightarrow}

\newcommand{\iprod}[2]{\left\langle#1,#2\right\rangle}

\newcommand{\biprod}[2]{\bigl\langle#1,#2\bigr\rangle}

\newcommand{\norm}[1]{\left\|#1\right\|}
\newcommand{\bnorm}[1]{\bigl\|#1\bigr\|}

\newcommand{\twonorm}[1]{\norm{#1}_2}
\newcommand{\infnorm}[1]{\norm{#1}_\infty}

\newcommand{\ev}[1]{\left\{#1\right\}}
\newcommand{\bev}[1]{\bigl\{#1\bigr\}}
\newcommand{\abs}[1]{\left|#1\right|}

\newcommand{\pa}[1]{\left(#1\right)}
\newcommand{\bpa}[1]{\bigl(#1\bigr)}
\newcommand{\Bpa}[1]{\Bigl(#1\Bigr)}

\newcommand{\wh}{\widehat}
\newcommand{\wt}{\widetilde}

\newcommand{\tP}{\wt{P}}

\newcommand{\tc}{\wt{c}}

\newcommand{\loss}{\ell}
\newcommand{\gen}{\textup{gen}}

\newcommand{\regret}{\mathrm{regret}}

\newcommand{\qed}{\hfill\BlackBox\\[2mm]}

\newcommand{\alg}{\mathcal{A}}
\newcommand{\Zw}{\mathcal{Z}}

\newcommand{\bH}{h}

\usepackage[normalem]{ulem}
\usepackage{enumitem}

\setcitestyle{numbers,square}

\author[Lugosi and Neu]{\Name[{G\'abor~Lugosi}]{G\'abor Lugosi} \Email{gabor.lugosi@gmail.com}\\
 \addr ICREA, Universitat Pompeu Fabra, and Barcelona School of Economics, Barcelona, Spain
 \AND
 \Name[{Gergely~Neu}]{Gergely Neu} \Email{gergely.neu@gmail.com}\\
 \addr Universitat Pompeu Fabra, Barcelona, Spain
 }
 
\begin{document}

\maketitle

\begin{abstract}
We present a new framework for deriving bounds on the generalization bound of statistical learning algorithms from the 
perspective of online learning. Specifically, we construct an online learning game called the ``generalization game'', 
where 
an online learner is trying to compete with a fixed statistical learning algorithm in predicting the sequence of 
generalization gaps on a training set of i.i.d.~data points. We establish a connection between the online and 
statistical learning setting by showing that the existence of an online learning algorithm with bounded regret in this 
game implies a bound on the generalization error of the statistical learning algorithm, up to a martingale 
concentration term that is independent of the complexity of the statistical learning method. This technique allows us 
to recover several standard generalization bounds including a range of PAC-Bayesian and information-theoretic 
guarantees, as well as generalizations thereof.
\end{abstract}

\begin{keywords}%
  statistical learning, generalization error, online learning, regret analysis
\end{keywords}

\section{Introduction}
We study the standard model of statistical learning.  We are given a
training sample of $n$ i.i.d.~data points $S_n = (Z_1,\dots,Z_n)$ drawn from a distribution $\mu$ over a measurable
\emph{instance space} $\Zw$. A
\emph{learning algorithm} maps the training sample to 
an output $W_n= \alg(S_n)$ taking values in a measurable set $\Ww$ (called the \emph{hypothesis class}) in a 
potentially randomized way.
{In other words,} a randomized learning algorithm assigns, to any
$n$-tuple of samples from $\Zw$, a probability distribution over $\Ww$
and draws a sample from that distribution, independently
of $S_n$. The resulting random element is denoted by $W_n$. {More precisely, denoting the set of distributions over 
hypotheses by 
$\Delta_{\Ww}$, a learning algorithm can thus be formally written as a
mapping $\alg: \Zw^n \to \Delta_{\Ww}$,
and $W_n=\alg(S_n)$.}

We study the performance of the learning algorithm measured by a \emph{loss function}
$\ell:\Ww\times\Zw\ra \real_+$. {Two key objects of interest are the \emph{risk} (or \emph{test error})
$\EE{\ell(w,Z')}$ and the \emph{empirical risk} (or \emph{training error}) $L(w,S_n) = \frac{1}{n}\sum_{i=1}^n 
\ell(w,Z_i)$  of a hypothesis $w\in\Ww$, 
where the random element $Z'$ has the same distribution as $Z_i$, and is independent of $S_n$. The ultimate goal in 
statistical learning is to find algorithms with small \emph{excess risk}
\begin{align*}
  \mathcal{E}(W_n) &= \EXP\left[ \ell(W_n,Z') | W_n\right] - \inf_{w\in \Ww} \EXP\left[ \ell(w,Z') \right],
\end{align*}
which is typically decomposed into the task of minimizing the empirical risk $L(\cdot,S_n)$, and showing that the 
gap between the true risk and the empirical risk is small.
This gap is commonly called the \emph{generalization error} of the algorithm, and is defined formally as
\[
 \gen(W_n,S_n)= \EEcc{\ell(W_n,Z')}{W_n} - L(W_n,S_n)~.
\]
}
In other words, the generalization error measures
the extent of \emph{overfitting} occurring during training. As such,
understanding (and more specifically, upper bounding) the generalization error 
has been in the center of focus of statistical learning theory ever since its inception. 
Over the past half century, numerous approaches have been proposed to tackle this challenge. Key ideas 
include uniform convergence arguments {as made in the Probably Approximately Correct (PAC) framework} 
\citep{VaCh74a}, distribution-dependent complexity measures like the Rademacher 
or Gaussian complexities \citep{BaBoLu01,Kol00a,BaMe02}, or various notions of stability that can guarantee small 
generalization error
\citep{devroye1979distribution,bousquet2002stability,mukherjee2006learning,shalev2010learnability}. 
The most relevant  to our work is the family of so-called PAC-Bayesian generalization bounds, a topic to which we 
return shortly \citep{STW97,McA98,Aud04,Cat07}.

In this work, we establish a connection between the statistical learning model described above and the model of online 
learning \citep{CBLu06:book,Ora19}. Online learning models sequential games between an online learner and its 
environment, 
where in each round $t=1,2,\dots,n$, the following steps are repeated: (1) the online learner picks a distribution over 
hypotheses $P_t\in\Delta_{\Ww}$; (2) the environment picks a cost function $c_t:\Ww\ra \real$; (3) the online learner 
incurs cost $\mathbb{E}_{\wt{W}_t\sim P_t}\bigl[c_t\bpa{\wt{W}_t}\bigr]$; (4) the online learner observes the cost 
function $c_t$. 
Importantly, the two players make their choices in parallel, and their actions are revealed to each other only at the 
end of the round. Typically, no assumptions about are made about the environment, and, in particular, it is allowed to 
have full knowledge of the online learner's decision-making
policy. The performance of an online learning algorithm $\Pi_n$ is then 
measured in terms of its \emph{regret} against a comparator point $P^*\in\Delta_{\Ww}$, defined as
\[
 \regret_{\Pi_n}(P^*) = \sum_{t=1}^n \pa{\EEs{c_t\pa{\wt{W}_t}}{\wt{W}_t\sim P_t} - \EEs{c_t\pa{W^*}}{W^*\sim P^*}}.
\]
Notably, the comparator $P^*$ is allowed to depend on the entire sequence of costs chosen by the environment, and is 
typically picked as the distribution minimizing the cumulative costs. In the last few decades, numerous algorithms with 
strong regret bounds have been proposed for the above setting and a variety of its generalizations.

{A well-known connection between online learning and statistical learning is the \emph{online-to-batch 
conversion} scheme \citep{CBCoGe04}, where one picks the cost functions as $c_t = \ell(\cdot,Z_t)$ in the above online 
learning game.
Then, by picking $P^*$ as a point mass on $\argmin_{w\in\Ww} \EE{\loss(w,Z')}$, it is straightforward to show 
that the excess risk associated with the random hypothesis $\overline{W}_n \sim \frac 1n \sum_{t=1}^n P_t$ 
satisfies
\[
 \mathcal{E}(\overline{W}_n,S_n) = \EE{\loss(\overline{W}_n,Z') - \inf_{w\in\Ww} \loss(w,Z')} = 
\frac{\sup_{P^*\in\Delta_{\Ww}} \regret_{\Pi_n}(P^*)}{n}.
\]
While this approach is very effective in that it directly controls the excess risk, it comes with the severe downside 
that achieving bounded regret against the worst-case comparator $P^*$ is hard from both a statistical and 
computational perspective. Indeed, guaranteeing low regret against worst-case comparators has been shown to require 
uniform convergence of the empirical risk to the true risk, and thus is only possible whenever classical complexity 
measures like the VC dimension or the Rademacher complexity of $\Ww$ are bounded \citep{RS17}. Additionally, algorithms 
achieving low regret need to optimize over function spaces (or, even worse, over the space of distributions over 
functions), which is computationally intractable in general.
}

{Our main contribution is establishing a new connection between online learning and statistical learning that does 
not suffer from these weaknesses. The most important difference is that our reduction does not directly bound the 
excess risk of an online algorithm, but rather the \emph{generalization error} of a \emph{given} statistical learning 
algorithm.
}
{The essence of our reduction is constructing an online 
learning game where the total cost accumulated by a comparator strategy that repeatedly plays $W_n$ (the output of the 
statistical learning algorithm) is equal to the generalization error. Then, we show that this game is ``hopeless'' for 
any online learning algorithm in that no matter what strategy the online learner follows, its expected total cost is 
always going to be zero. Thus, the only way for an online learner to have small regret against the comparator 
strategy in question is for the generalization error itself to be small.
By making this reasoning formal, we show that the 
\emph{existence} of an online learning algorithm with small regret implies a bound on the generalization error. 
Notably, this means that our approach does not suffer from computational limitations since the online algorithm does 
not have to be implemented; it is merely an artifice used for proving statistical bounds on the generalization error.}

{The nature of the guarantees that can be derived from the resulting framework is closely related to the so-called 
family of \emph{PAC-Bayesian} generalization bounds \citep{STW97,McA98,Aud04,Cat07}. Similarly to PAC guarantees that 
provide bounds that hold uniformly over a class of 
hypotheses, PAC-Bayesian guarantees provide bounds that hold uniformly over the \emph{class of all 
probability distributions} supported on a given hypothesis class. Our conversion framework provides 
analogous results: bounds on the generalization error that hold uniformly over the set of hypothesis distributions.
To emphasize the connection between PAC learning, online learning, and online-to-batch conversions, we refer to our 
framework as \emph{online-to-PAC conversion}.}

{
 We emphasize that, while we consider randomized learning algorithms, our framework 
 allows one to study statistical learning algorithms with
 deterministic outputs.
 In fact, we establish some results 
 that apply to such methods, such as the bounds stated in Section \ref{sec:wasserstein}
 that depend on the Wasserstein distance 
 which can be bounded even for degenerate prior and posterior
 distributions. 
}

Various other links have been previously established between online learning and a variety of concentration 
inequalities. It 
is 
well known that the relatively simple setting of testing and mean estimation of scalar-valued random variables is
intimately connected with sequential betting---see, e.g., \citet{SV01}, \citet{WSR20}, and especially \citet{OJ21} for 
an impressive literature review on the subject. Our work draws direct inspiration from \citet*{KST08}, who used online 
learning techniques for bounding the Rademacher complexity of linear function classes. 
An even earlier example of a similar flavor is the work of \citet{Zha02}, who proved bounds on the covering numbers of 
linear function classes via a reduction to online binary classification and in particular the classical Perceptron 
mistake 
bound (cf.~his Theorem~4). In the broader context of concentration of empirical processes, \citet{RS17} showed an 
\emph{equivalence} between tail inequalities on the supremum of a collection of martingales and 
the existence of an online learning algorithm with bounded regret in a game tailor-made to the collection of martingales 
in 
question. While many of these results can be adapted to provide bounds on the generalization error of statistical 
learning methods, they all suffer from being overly conservative by considering a supremum over a collection of random 
variables, which only serves as a possibly loose proxy to the quantity of interest. Our key observation is that the 
general approach taken in the above works can be significantly strengthened when specialized to statistical 
learning, and online learning techniques can be employed to bound the generalization error in an algorithm-dependent 
fashion rather than the unnecessarily pessimistic worst-case fashion.\looseness=-1

Our online-to-PAC conversion scheme allows us to recover a range of previously known generalization bounds, as well as 
to establish some new ones. As our most elementary example, we show that the classical PAC-Bayesian generalization 
bound 
of 
\citet{McA98} can be directly recovered from our framework by employing the standard exponentially weighted average 
algorithm of \citet{LW94} as the online learning method in the generalization game. To illustrate the power 
of our reduction, we derive a variety of extensions to this fundamental theorem, including data-dependent bounds 
that approach zero at a fast rate when the empirical risk is zero, and a parameter-free bound that shaves off a 
logarithmic factor that appears in all other PAC-Bayesian bounds that we are aware of. We also provide a much more 
general family of generalization bounds that replaces the relative entropy appearing in the classical PAC-Bayesian 
bounds 
with an appropriately chosen strongly convex function of the conditional distribution of the output $W_n$ given the 
input $S_n$, via an application of the standard Follow-the-Regularized-Leader (FTRL) algorithm for online learning 
(see, 
e.g., \citet{Ora19} and the references therein). Furthermore, we provide an empirical version of the 
latter bounds by adapting the idea of ``optimistic'' updates as proposed by \citet{RS13,RS13b}. As an example 
application of these techniques, we provide a new generalization bound that replaces the relative entropy 
factor and the subgaussianity constant appearing in the classical PAC-Bayes bounds with the squared Wasserstein-2 
distance, and a Sobolev-type norm of the loss function. Finally, we provide an extension to our framework that allows 
the use of data-dependent priors and regularizers in the vein of the ``almost exchangeable priors'' of \citet{Aud04} 
and \citet{Cat07}, which also allows us to recover several classical PAC-learning bounds.

{Finally, let us comment on our terminology. We note that the term ``generalization error'' is sometimes used in a 
different sense from the way it is used in this paper, and in particular it is often used to refer to the excess risk 
$\mathcal{E}(W_n,S_n)$
instead. However, recent literature on statistical learning theory
focuses primarily on bounding the gap between the empirical and the test risk
(i.e., the training and the test error). 
The main reason for this shift of focus is that 
 in modern machine learning, the hypothesis classes (e.g., classes of
 functions realizable by large neural networks)
 are enormous and therefore it is hopeless to obtain  
meaningful excess-risk bounds in a distribution-free manner \citep{NK19}. 
On the other hand, many learning algorithms 
 are able to fit the data nearly perfectly, and in such cases low
 generalization error immediately implies low excess risk.
Also, computable upper bounds for the generalization error that hold
with high probability yield confidence intervals for the population
risk, while bounds for the excess risk do not imply such
confidence bounds.
Our usage of the term ``generalization error'' is in line with the recent literature on the subject 
\citep{DR17,NBMS17,NBS18,NK19,JNMKB20}, as well as with all the PAC-Bayesian literature discussed above. 
}

The rest of the paper is organized as follows. In Section~\ref{sec:o2PAC}, we provide further details about the setup 
we consider and introduce our online-to-PAC conversion framework. In Section~\ref{sec:applications}, we provide a long 
list of applications of the framework, including a variety of PAC-Bayesian generalization bounds in 
Section~\ref{sec:PAC-Bayes}, and generalized versions thereof in Section~\ref{sec:FTRL}. We provide further extensions 
of the framework in Section~\ref{sec:extensions} and conclude in Section~\ref{sec:conc}. Additional proofs are deferred 
to Appendices~\ref{app:regret} and~\ref{app:tools}, provided as  supplementary material. \looseness=-1

\textbf{Notation.} For a distribution over hypotheses $P\in\Delta_{\Ww}$ and a bounded measurable function 
$f:\Ww\ra\real$, we write $\iprod{P}{f}$ to refer to the expectation of $\EEs{f(W)}{W\sim P}$. {We denote the set 
of all signed measures over $\Ww$ by $\mathcal{Q} = \ev{a P- b P': P,P'\in\Delta_{\Ww}, a,b\in\real}$, and equip it 
with a norm $\norm{\cdot}$. The two objects together form the Banach space $\pa{\mathcal{Q},\norm{\cdot}}$. We use 
$\norm{\cdot}_*$ to denote the corresponding dual norm on the dual space 
$\mathcal{Q}^*$ of measurable functions on $\Ww$, defined for each $f\in\mathcal{Q}^*$ as $\norm{f}_* = 
\sup_{Q\in\mathcal{Q}: \norm{Q}\le 1} \iprod{Q}{f}$.
}

\section{Online-to-PAC Conversions}\label{sec:o2PAC}
Our framework relies on the construction of an online learning game that we call \emph{generalization game} associated 
with the statistical learning problem at hand. The game is defined as 
a sequential interaction scheme between an online learner and an environment, where the 
following steps are repeated in a sequence of rounds $t=1,2,\dots,n$:
\begin{enumerate}
 \item the online learner picks a distribution $P_t \in \Delta_{\Ww}$;
 \item the environment selects a cost function $c_t: \Ww\ra \real$ defined for each $w\in\Ww$ as
 \[
  c_t(w) = \ell(w,Z_t) - \EEs{\ell(w,Z')}{Z'\sim\mu}~;
 \]
 \item the online learner incurs cost $\iprod{P_t}{c_t}= \EEs{c_t(W)}{W\sim P_t}$;
 \item the environment reveals the value of $Z_t$
   to the online learner.
 \end{enumerate}
{Note that, since the data points $Z_t$ are i.i.d., the environment is oblivious to the moves of the online 
learner.} Furthermore, the online learner is allowed to know the loss function $\ell$ and the 
distribution $\mu$ of the data points $Z_t$, and therefore by revealing the value of $Z_t$, the online learner may 
compute the entire cost function $c_t(\cdot)$. However, the online learner is \emph{not} allowed any type of access to 
the realization of the future data points $Z_t,\dots, Z_n$ before making its decision about $P_t$.
 
We use $\F_{t}$ to denote the sigma-algebra induced by the sequence of random variables generated and used by both 
players (including the random data points $Z_1,\dots,Z_t$ and all potential randomization utilized by 
the online learner) up until the end of round $t$. 
Formally, an \emph{online learning algorithm} $\Pi_n = \{P_t\}_{t=1}^n$ is a
  sequence of  functions such that $P_t$ maps the sequence of past
  outcomes $(z_1,\ldots,z_{t-1})\in \Zw^{t-1}$ to $\Delta_{\Ww}$, the
  set of all probability distributions over the hypothesis class
  $\Ww$. For the brevity of notation, we abbreviate
  $P_t=P_t(Z_1,\ldots,Z_{t-1})$. We denote by $\Pw_n$ the class of all
  online learning algorithms over sequences of length $n$.\looseness=-1

A random variable of crucial importance, associated to the online
  algorithm $\Pi_n$, is
  \[
     M_{\Pi_n}= \frac 1n \sum_{t=1}^n \iprod{P_t}{c_t}~.
   \]
Notice that for any online learning algorithm, $M_{\Pi_n}$ is a
normalized sum of martingale differences, due to the conditional independence of $P_t$ 
and $c_t$:
\begin{align*}
 \EEcc{\iprod{P_t}{c_t}}{\F_{t-1}} &= \mathbb{E}_{\wt{W}_t\sim P_t}\bigl[c_t\bpa{\wt{W}_t}\big| \F_{t-1}\bigr]
 \\
 &= \mathbb{E}_{\wt{W}_t\sim P_t}\bigl[\loss\bpa{\wt{W}_t,Z_t} - \loss\bpa{\wt{W}_t,Z'}\big| \F_{t-1}\bigr] = 0~.
\end{align*}
Indeed, the online learning protocol guarantees that $\wt{W}_t$ is chosen before $Z_t$ is revealed to the online 
learner, and thus $(\wt{W}_t,Z_t)$ has the same conditional distribution as $(\wt{W}_t,Z')$ given the history 
$\F_{t-1}$.

As it is usual in online learning, the goal of the online learner is to 
accumulate a total cost that is not much worse  than an appropriately
chosen {(possibly degenerate)} comparator distribution 
$P^*\in\Delta_{\Ww}$. Specifically, the performance metric for evaluating the online learner's performance is the 
\emph{regret} defined against the comparator $P^*$ as 
\[
 \regret_{\Pi_n}(P^*) = \sum_{t=1}^n \iprod{P_t - P^*}{c_t}~.
\]
Many online learning algorithms come with performance guarantees that hold with probability one for all 
cost sequences $c_t$, against comparators $P^*$ that may depend on the cost sequence in an arbitrary way. We 
exploit this property below by choosing a very specific comparator point that allows us to establish our main 
result.

In particular, we choose the comparator point $P^*$ as the conditional distribution of the output 
$W_n$ given the input $S_n$.
We denote this distribution by $P_{W_n|S_n}$ and remark that when the statistical learning algorithm does not 
use randomization, then the distribution $P_{W_n|S_n}$ puts all its mass to a single point $W_n$. For randomized 
algorithms, we study the generalization error in expectation with respect to the additional randomization of $W_n|S_n$.
To this end, we introduce
\[
 \bgen = \EXP  \left[ \gen(W_n,S_n)|  S_n \right].
\]
Naturally, for algorithms that do not use randomization, we have $\bgen = \gen(W_n,S_n)$.

The following theorem characterizes the generalization error in terms of the regret of an online learning algorithm in 
the generalization game.
\begin{theorem}\label{thm:main}
The generalization error of any learning algorithm $W_n = \alg(S_n)$
satisfies that, for any online learning algorithm $\Pi_n \in \Pw_n$,
\[
 \bgen = 
 \frac{\regret_{\Pi_n}(P_{W_n|S_n})}{n} - M_{\Pi_n}~.
\]
\end{theorem}
\begin{proof}
Let $\Pi_n=\{P_t\}_{t=1}^n$ be an arbitrary online learning algorithm. 
Recalling the notation $\iprod{P}{f} = \EEs{f(W)}{W\sim P}$, we rewrite the conditional expectation of the 
generalization error as follows:
\begin{align*}
\bgen &= {\frac 1n \sum_{t=1}^n \EEcc{\ell(W_n,Z') - \ell(W_n,Z_t)}{S_n}  }
= - \frac 1n \sum_{t=1}^n \EEcc{c_t(W_n)}{S_n}
\\
&= - \frac 1n \sum_{t=1}^n \iprod{P_{W_n|S_n}}{c_t} 
\\
&= \frac 1n \sum_{t=1}^n \iprod{P_t - P_{W_n|S_n}}{c_t}
- \frac 1n \sum_{t=1}^n \iprod{P_t }{c_t}
\\
&= \frac {\regret_{\Pi_n}(P_{W_n|S_n})}{n} - \frac 1n \sum_{t=1}^n \iprod{P_t}{c_t}~.
\end{align*}
Recalling the definition $M_{\Pi_n} = \frac 1n \sum_{t=1}^n \iprod{P_t}{c_t}$ concludes the proof.
\end{proof}

While the claim of Theorem~\ref{thm:main} and its proof are strikingly simple (and perhaps even trivial in hindsight), 
it has never appeared in previous literature---at least to our knowledge. Also, the result has several 
implications that may not be obvious on the first look. Below, we provide a list of comments to illuminate some of 
the most interesting aspects of Theorem~\ref{thm:main}.

\medskip
\noindent
\textbf{PAC-style bounds on the generalization error.}
We may use the simple observation of Theorem \ref{thm:main} in various
ways to obtain upper bounds for the generalization error. The simplest way
is to choose a single online learning algorithm and use martingale
concentration arguments to bound the second term on the right-hand side. Then
one may invoke known regret bounds for specific online learning
algorithms. For example, the following corollary is immediate.
\begin{corollary}
\label{cor:single}
Consider an arbitrary online learning algorithm $\Pi_n \in \Pw_n$ and suppose 
  that there
  exists $\sigma>0$ such that $\sup_{w\in \Ww} \EE{\ell(w,Z')^2} \le \sigma^2$.
  Then, with probability at least $1-\delta$,
  the generalization error of all statistical learning algorithms $W_n=\alg(S_n)$ simultaneously satisfy
\[
 \bgen \le
 \frac{\regret_{\Pi_n}(P_{W_n|S_n})}{n}
 + \sqrt{\frac{2\sigma^2
\log\pa{\frac{1}{\delta}}}{n}}~.
\]
\end{corollary}
{Notably, the bound holds with high probability \emph{simultaneously for all 
randomized statistical learning algorithms} against which the online learner has bounded regret. This property mirrors 
the key property of PAC-Bayesian generalization bounds that hold uniformly over ``posteriors'' in the same sense, which 
thus justifies the name ``online-to-PAC conversion''.}
The proof follows from applying a standard concentration result for the lower tail of a sum of nonnegative 
random variables {(which is applicable since we assume the losses to be nonnegative)}.
In particular, we note that
 \[
  \EXP_{Z_t \sim \mu}  \iprod{P_t}{\ell(\cdot,Z_t)}^2 
=  \EXP_{Z_t \sim \mu} \left[\EXP_{W \sim P_t} \ell(W,Z_t)\right]^2
\le \sigma^2~,
\]
and use a standard result (stated as Lemma~\ref{lem:concentration} in Appendix~\ref{app:concentration}) to upper bound 
$-M_{\Pi_n}$.
Notice that this argument does not require the loss function to be uniformly bounded from above, or have light 
tails, yet it still yields a subgaussian bound for the lower tail of the martingale $M_{\Pi_n}$.  
Alternatively, one may assume that that the loss function is (uniformly) subgaussian in the
sense that there exists $\Sigma>0$ such that for all $\lambda <0$,
\[
\sup_{w\in \Ww}  \EXP_{Z'\sim \mu} \left[
  e^{\lambda(\ell(w,Z')-\EXP_{Z\sim \mu} \ell(w,Z))}\right] \le e^{\lambda^2\Sigma^2/2}~.
 \]
 In that case, the second term in the upper bound of Corollary~\ref{cor:single}
 becomes $\sqrt{\frac{\Sigma^2 \log\pa{\frac{1}{\delta}}}{2n}}$, by
 standard exponential martingale inequalities. 
 An extension of Corollary~\ref{cor:single} to bounds that hold uniformly over time can be easily 
achieved by noting that online learning algorithms typically come with time-uniform bounds that hold with 
probability one over all data sequences, and thus it only remains to gain uniform control over the martingale term, 
which can be done via standard techniques.

\medskip
\noindent
\textbf{Classes of online learning algorithms.}
The results above suggest a strategy to obtain upper bounds for the generalization error: choose an 
appropriate online learning algorithm and bound its regret against $P_{W_n|S_n}$. 
We show several examples on the remaining pages of this paper. However, sometimes one may obtain better bounds by 
considering
a class of online learning algorithms. For example, if one chooses a
finite class $\C_n$ of online learning algorithms, the argument of
Corollary~\ref{cor:single}, combined with the union bound, immediately
implies the following.

\begin{corollary}
\label{cor:multiple}
Consider a finite set $\C_n\subset \Pw_n$ of $N$ online learning algorithms and suppose that there exists $\sigma>0$ 
such that $\sup_{w\in \Ww} \EXP_{Z'\sim \mu} \ell(w,Z')^2 \le \sigma^2$.
Then, with probability at least $1-\delta$,
  the generalization error of all statistical learning algorithms $W_n=\alg(S_n)$ simultaneously satisfy  
\[
 \bgen \le
\min_{\Pi_n\in \C_n}  \frac{\regret_{\Pi_n}(P_{W_n|S_n})}{n}
 + \sqrt{\frac{2\sigma^2
\log\pa{\frac{N}{\delta}}}{n}}~.
\]
\end{corollary}

The reason why it is often beneficial to consider larger classes of
online learning algorithms is that the regret bounds may be different
for different training samples and Corollary \ref{cor:multiple} allows
us to use the best of these bounds. 
Corollary \ref{cor:multiple} is based on the naive upper bound
\[
  \inf_{\Pi_n \in \C_n} \left(
 \frac{\regret_{\Pi_n}(P_{W_n|S_n})}{n} 
 - M_{\Pi_n} \right)
\le   \inf_{\Pi_n \in \C_n} 
 \frac{\regret_{\Pi_n}(P_{W_n|S_n})}{n} + \sup_{\Pi_n\in \C_n} \left(- M_{\Pi_n} \right)~,
\]
and the second term on the right-hand side in Corollary \ref{cor:multiple}
is simply an upper bound for
$\max_{\Pi_n\in \C_n}  (- M_{\Pi_n})$.
The price for the minimum over $\C_n$ is an additional term of order $\sigma
\sqrt{(\log N)/n}$ in the upper bound.
In principle, 
one may be able to improve on this naive bound by a more careful balancing
of the regret and martingale terms.
For example, if $\Pi_n^*$ denotes an online learning algorithm that
(approximately) minimizes $\regret_{\Pi_n}(P_{W_n|S_n})$ over
$\Pi_n\in \C_n$, then it suffices to bound $- M_{\Pi^*_n}$. Since
$\Pi^*_n$ depends on the random sample, obtaining sharp upper bounds
for $- M_{\Pi^*_n}$ is a nontrivial matter. 
In some cases, one may be
able to improve on the simple union bound that we used above. 
This requires a deeper understanding of the martingale process
$\{M_{\Pi_n}: \Pi_n\in \C_n\}$. We leave further study of this interesting 
question for
future research.  

\medskip
\noindent \textbf{Comparator-dependent vs.~worst-case bounds.} The result closest in spirit to our 
Theorem~\ref{thm:main} in the previous literature that we are aware of is Theorem~1 of
\citet{KST08}. This result bounds the Rademacher complexity of linear function classes using a technique similar to 
ours, and in fact their proof served as direct inspiration for our proof above. The key difference between their 
approach and ours is that we directly bound the generalization error of a given algorithm $\alg$ via an 
algorithm-specific choice of comparator point, as opposed to the simple worst-case choice taken by \citet*{KST08}. In 
particular, this choice allowed us to go beyond uniform-convergence-type arguments, which are known to be insufficient 
to understand the statistical behavior of modern machine learning algorithms \citep{NK19,JNMKB20}.

\medskip
\noindent
\textbf{The online learner does not learn.}
In words, the key idea of the online-to-PAC conversion scheme is the following. In the generalization game, the cost 
function selected by the environment corresponds to the generalization gap on example $Z_t$, and has zero expectation 
for every fixed $w\in\Ww$. 
{As such, no matter what strategy it follows, the online learner's cost has zero expectation in each round, and its 
total cost is thus a martingale. On the other hand, with probability 1, the online learner is guaranteed to have 
bounded regret against a data-dependent comparator strategy that always plays the output of the statistical learning 
algorithm. The only way for both statements to be true is for the total cost of the comparator strategy to be close to 
zero---but this implies a bound on the generalization error, which is equal by construction to the total cost of the 
comparator.} Thus, the implication of the result is that whenever one can prove the \emph{existence} an online learner 
that predicts the sequence of i.i.d.~generalization gaps of the statistical learning method well, the statistical 
learning algorithm can be guaranteed to generalize well to test data. {We stress that this online learning 
algorithm never has to be executed on the data sequence, as proving its existence is sufficient for bounding the 
generalization error.}

Notably, the goal of the online learner in the generalization game is \emph{not} to achieve low total cost, but rather 
to achieve a cumulative cost that is \emph{comparable} to the generalization error of the statistical learning method. 
Indeed, since the costs are all zero-mean, minimizing the cost is a hopeless task by definition, and the online learner 
can only hope to do ``not much worse'' than an ideal comparator that may depend on the realization of the data 
sequence. {The abundant literature on online learning provides a huge array of algorithmic tools that achieve this 
goal for \emph{any} realization of the data sequence, without making statistical assumptions \citep{CBLu06:book,Ora19}. 
In 
this sense, the regret decomposition of Theorem~\ref{thm:main} splits the generalization error into two main terms: a 
martingale that can be controlled via elementary tools from probability, and another term that can be controlled via 
purely algorithmic tools, without invoking any probabilistic machinery.}

\medskip
\noindent
\textbf{All online learners are equivalent.}
We highlight that the relationship between the generalization error, the regret, and the total cost of the 
learner holds with \emph{equality for all online learners}, which may appear counterintuitive at first sight. 
This implies that 
whenever the generalization error $\bgen$ decays to zero as $n$ increases, \emph{all online learning algorithms have 
  sublinear regret against $P_{W_n|S_n}$}
(up to the martingale term whose fluctuations are controlled), due to both sides of the bound of Theorem~\ref{thm:main} 
necessarily vanishing 
for large $n$.
Likewise, if one can establish a sublinear regret bound
for \emph{any} particular online learning algorithm, then it implies that the regret 
of \emph{all other} online learning methods are also sublinear in this game (up to martingale fluctuations). 
{While counterintuitive at first, this phenomenon does have a simple explanation: since the generalization game is 
equally hopeless for all online learners, the fact that one online learner does well in comparison to the comparator 
must mean that all others also do well in the same sense.}
{Due to this equivalence between regrets, it is thus often more meaningful to apply 
Theorem~\ref{thm:main} with the infimum being taken over \emph{upper bounds} on the 
regrets rather than the the regrets themselves (at least as long as one is interested in proving upper bounds on the 
generalization error, as we are in the present paper). Indeed, by showing that the infimum among upper bounds is small, 
one certifies that the generalization error and thus the regret of all online algorithms is small.}
As an aside, we note that the observation about the regrets being all equal can be 
used to prove \emph{lower bounds} on the regret of online learning algorithms. In fact, our decomposition is 
closely related to a common technique for proving lower bounds via ``Yao's minimax principle'' \citep{Yao77}.

\section{Applications}\label{sec:applications}

In what follows, we instantiate the general bound of Theorem~\ref{thm:main} using a variety of concrete choices for the 
online learning algorithm. For the sake of completeness, we include the proofs of the regret bounds we make use of in 
Appendix~\ref{app:regret}. The purpose of this section is not necessarily to provide results that beat the 
state of the art, but rather to illustrate the use of our framework and demonstrate its flexibility. In 
Section~\ref{sec:PAC-Bayes} we start with deriving some classical PAC-Bayes-flavored generalization bounds, including 
some 
well-known ones and some others that we have not seen in the related literature. The impatient reader interested in 
going beyond PAC-Bayes can skip ahead to Section~\ref{sec:FTRL} that includes a range of new generalization guarantees 
derived from a general family of online learning methods known as ``follow-the-regularized-leader''.

\subsection{PAC-Bayesian generalization bounds via multiplicative weights}\label{sec:PAC-Bayes}
The most elementary application of our framework is based on using the
classical exponential weighted average (EWA)---or multiplicative weights---algorithm of \citet{LW94}
as the online learner's strategy in the generalization game (see also \citet{Vov90} and \citet{FS97}).
The most basic version of this method is initialized with some fixed ``prior'' 
distribution $P_1 \in \Delta_{\Ww}$, and then sequentially calculates its updates by solving the optimization 
problem
\[
 P_{t+1} = \argmin_{P\in\Delta_{\Ww}} \ev{\iprod{P}{c_t} - \frac{1}{\eta} \DDKL{P}{P_t}}~,
\]
where $\eta$ is a positive \emph{learning-rate} parameter, and $\DDKL{P}{Q}$ is the relative entropy between the 
distributions $P$ and $Q$. {Under mild regularity conditions, the minimizer can be shown to exist and satisfies}
\[
 \frac{\dd P_{t+1}}{\dd P_t}(w) = \frac{e^{-\eta c_t(w)}}{\int_{\Ww} e^{-\eta c_t(w')} \dd P_t(w')}~.
\]
In what follows, we derive a range of generalization bounds using the classical regret analysis of this algorithm and 
some of its variants. {For most (although not all) results in this section}, we suppose that the loss function is 
uniformly bounded in the range 
$[0,1]$. We note that a relaxation to more general subgaussian losses is also possible at the expense of a slightly 
more involved technical analysis.

\subsubsection{A vanilla PAC-Bayes bound}
A direct application of the classical regret analysis of the multiplicative weights algorithm gives the following 
generalization bound via the reduction of Corollary~\ref{cor:single}:
\begin{corollary}\label{cor:MWA_untuned}
Suppose that there   exists $\sigma>0$ such that $\sup_{w\in \Ww} \EE{\ell(w,Z')^2} \le \sigma^2$.
  Then, for any fixed $\eta > 0$, any $P_1\in\Delta_{\Ww}$ and any $n>1$,
with probability at least $1-\delta$,
  the generalization error of all statistical learning algorithms $W_n=\alg(S_n)$ simultaneously satisfy  
\[
 \bgen \le \frac{\DDKL{P_{W_n|S_n}}{P_1}}{\eta n} + \frac{\eta}{2n} \sum_{t=1}^n 
\bnorm{\ell(\cdot,Z_t) - \EE{\ell(\cdot,Z')}}_\infty^2 + 
\sqrt{\frac{\sigma^2 \log(1/\delta)}{2n}}~.
\]
\end{corollary}
Note that the first two terms appearing on the right hand side depend on the realization of the data set $S_n$, and are 
thus random variables. Also notice however that the second term is not empirically observable as it features the test 
error $\EE{\ell(\cdot,Z')}$. Under the additional assumption that the loss function $\ell$ is almost surely bounded in 
$[0,1]$, the data-dependent quantity $\frac{1}{n} \sum_{t=1}^n \infnorm{\ell(\cdot,Z_t) - \EE{\ell(\cdot,Z')}}^2$ can 
be simply bounded by $1$. This result essentially recovers the original PAC-Bayes bound of \citet{McA98}. The proof of 
the regret bound serving as the foundation of our result is included in Appendix~\ref{app:MWA}. 

Notice that the bound of Corollary~\ref{cor:MWA_untuned} only holds for a fixed $\eta$, 
and optimizing the bound requires choosing a data-dependent $\eta$ due to the randomness of $\DDKL{P_{W_n|S_n}}{P_1}$. 
Such learning-rate choices are disallowed by our framework, as they would require the learner to access information 
about future data points $Z_t,\dots,Z_n$ when picking its decision $P_t$. However, we may use the
second statement of Theorem~\ref{thm:main} with $\C_n$ containing 
exponentially weighted average algorithms with a given prior $P_1$ and a range of different values of the learning 
parameter $\eta$. In particular, an application of Corollary~\ref{cor:multiple} on an appropriately chosen finite 
range of learning rates gives the following result.
\begin{corollary}\label{cor:MWA_tuned}
  Suppose that $\ell(w,z)\in[0,1]$ for all $w,z$.
Fix $\epsilon \in (0,1]$.
  Then, for any
$P_1\in\Delta_{\Ww}$ and any $n>1$,
with probability at least $1-\delta$,
  the generalization error of all statistical learning algorithms $W_n=\alg(S_n)$ simultaneously satisfy
\[
  \bgen \le
      \left(1+\frac{\epsilon^2}{2}\right) \sqrt{\frac{\DDKL{P_{W_n|S_n}}{P_1}}{2n}}
  +
\sqrt{\frac{ \log \log (4\sqrt{n}) + \log(1/\delta) + \log \frac{2}{\epsilon} }{2n}}~.
\]
\end{corollary}

We include the simple proof for didactic purposes below.
\begin{proof}
Recall that if the exponentially weighted average algorithm is run
with tuning parameter $\eta$, then its regret may be upper bounded by
\[
  \frac{\DDKL{P_{W_n|S_n}}{P_1}}{\eta n} + \frac{\eta}{8}~.
\]
Observe first that the upper bound is trivial whenever $\frac{\eta}{8} >
1$. To proceed, denote the optimal (data-dependent) learning rate as
$\eta^* = \sqrt{8\DDKL{P_{W_n|S_n}}{P_1}/n}$.
Running the exponentially weighted average algorithm with this
choice of the tuning parameter yields the regret bound
\[
 \frac{\DDKL{P_{W_n|S_n}}{P_1}}{\eta^* n} + \frac{\eta^* }{8} 
= \sqrt{\frac{\DDKL{P_{W_n|S_n}}{P_1}}{2n}}~.
\]
If $\eta^* < 2/\sqrt{n}$, then $\DDKL{P_{W_n|S_n}}{P_1} \le 1/2$ the regret bound is at most $1/\sqrt{4n}$,
which is absorbed by the second term in the stated bound.   
Hence, it is sufficient to consider learning rates that 
are at least as large as $2/\sqrt{n}$.

Let $a>0$ and let $\C_n$ contain all exponentially weighted average
online learning algorithms with prior $P_1$ and learning rate
$\eta \in \{ a^i: i\in \NN\} \cap [2/\sqrt{n},8]$.
By Corollary \ref{cor:multiple}, with probability at least $1-\delta$,
for all $W_n=\alg(S_n)$, we have
\[
  \bgen \le
\min_{\eta \in \{ a^i: i\in \NN\} \cap [2/\sqrt{n},8]}
\left(\frac{\DDKL{P_{W_n|S_n}}{P_1}}{\eta n} + \frac{\eta}{8} \right)+
\sqrt{\frac{ \log \frac{\log_a(2\sqrt{n})}{\delta}}{2n}}~.
\]
Since the optimal choice of $\eta$ in the set $ \{ a^i: i\in \NN\}$ is
at most a factor of $a$ away from $\eta^*$, we get that
\[
\min_{\eta \in \{ a^i: i\in \NN\} \cap [2/\sqrt{n},8]}
\left(\frac{\DDKL{P_{W_n|S_n}}{P_1}}{\eta n} + \frac{\eta}{8} \right)
\le \frac{a+1/a}{2} \sqrt{\frac{\DDKL{P_{W_n|S_n}}{P_1}}{2n}}~,
\]
yielding the bound
\[
  \bgen \le
      \frac{a+1/a}{2} \sqrt{\frac{\DDKL{P_{W_n|S_n}}{P_1}}{2n}}
  +
\sqrt{\frac{ \log \log(4\sqrt{n}) + \log(1/\delta) - \log\log a }{2n}}~.
\]
Now we may choose the value of $a$. If $a=1+\epsilon$ for some
$\epsilon \in (0,1]$, then $(a+1/a)/2
\le 1+\epsilon^2/2$. Moreover, $\log a \ge \epsilon/2$, which implies
the stated inequality.
\end{proof}

\subsubsection{A parameter-free bound}\label{sec:parameter-free}
The bound of Corollary~\ref{cor:MWA_tuned} is derived from aggregating a number of different parametric bounds over 
a grid of learning rates $\eta$. This procedure adds a $\sqrt{\frac{\log\log n}{n}}$ term to the generalization bound 
via a union bound. It is natural to ask if it is possible to avoid this overhead by tapping into the literature on 
\emph{parameter-free online learning} algorithms that avoid using learning rates altogether 
\citep{CFH09,CV10,LS15,KvE15,OP16}. Due to the flexibility of our online-to-PAC framework, this question can be easily 
answered in the positive. The following theorem instantiates a generalization bound that can be derived from 
Corollary~6 
of \citet{OP16} (or, equivalently, Theorem~9 of \citet{HEK18}).
\begin{corollary}
Suppose that $\ell(w,z)\in[0,1]$ for all $w,z$.
Then, for any $P_1\in\Delta_{\Ww}$, with probability at least $1-\delta$,
  the generalization error of all statistical learning algorithms $W_n=\alg(S_n)$ simultaneously satisfy 
\[
\bgen \le \sqrt{\frac{3\DDKL{P_{W_n|S_n}}{P_1}+9}{n}}  +\sqrt{\frac{\log \frac 
{1}{\delta}}{2n}}.
\]
\end{corollary}
Notably, this bound does not feature any logarithmic factors of $n$. We are not aware of PAC-Bayesian guarantees with 
such property, with the exception of Theorem~1 of \citet{McA13} that only holds for countable hypothesis classes. We 
find it plausible that similar guarantees can be derived from the other works we have listed above. In particular, 
Theorem~9 of \citet{CV10} provides essentially the same regret guarantee as the one we have used above, and 
additionally holds uniformly over time.

\subsubsection{A data-dependent bound}\label{sec:second-order}
We now provide a tighter data-dependent bound derived from a slightly more sophisticated version of the standard 
exponentially weighted average forecaster. Under some conditions, this bound will decay to zero at a fast rate, which 
necessitates some adjustments to the basic setup that will allow us to bound the martingale term  $-M_{\Pi_n}$ more 
effectively. In particular, we will employ the skewed cost surrogate $c_t(w) = \ell(w,Z_t) - 
\EE{\ell(w,Z')} + \eta\pa{\ell(w,Z_t)}^2$, which turns $-M_{\Pi_n}$ into a supermartingale that can be upper-bounded 
much more tightly. The price paid for this adjustment is a data-dependent term appearing in the generalization bound 
that will be shown to be strictly dominated by the regret bound arising from the online-to-PAC conversion we will 
employ. We note that this technique can be used more generally to achieve fast rates for other algorithms. 

The online learning algorithm itself will make use of two tricks familiar from the 
online learning literature: optimistic updates as introduced by \citet{RS13,RS13b} and second-order adjustments as used 
in several works on adaptive online learning (e.g., \citet{CBMS07,gaillard2014a-second-order,KvE15}). In particular, 
the online learning algorithm uses a \emph{guess} $g_t\in\real^{\Ww}$ of the cost function $c_t$, and calculates 
two sequences of updates. The first is a sequence of auxiliary distributions initialized at $\tP_1$ and updated as 
\[
 \frac{\dd \tP_{t+1}}{\dd \tP_t}(w) = \frac{e^{-\eta c_t(w) - \eta^2 (c_t(w) - g_t(w))^2}}{\int_\Ww e^{-\eta c_t(w') - 
\eta^2 (c_t(w') - g_t(w'))^2} \dd \tP_t (w')},
\]
and the main update is calculated as
\[
 \frac{\dd P_{t+1}}{\dd \tP_{t+1}}(w) = \frac{e^{-\eta g_{t+1}(w)}}{{\int_\Ww e^{-\eta g_{t+1}(w')} \dd \tP_{t+1} 
(w')}}.
\]
We state a regret guarantee for this algorithm in Appendix~\ref{app:second-order}.
Instantiating the method with the 
choice $g_t(w) = -\EE{\ell(w,Z')}$ and plugging the regret bound into our online-to-PAC 
conversion scheme gives the following generalization bound:

\begin{corollary}
Suppose that $\ell(w,z)\in[0,1]$ for all $w,z$. For any $\tP_1\in\Delta_{\Ww}$ and any $\eta \in \left[0,\frac 
14\right]$, with probability at least $1-\delta$, the generalization error of all statistical learning algorithms 
$W_n=\alg(S_n)$ simultaneously satisfy
\[
  \bgen \le \frac{\DDKLb{P_{W_n|S_n}}{\tP_1}}{\eta n} + \frac{3\eta}{n}\sum_{t=1}^n 
\EEcc{\bpa{\ell(W_n,Z_t)}^2}{S_n} + \frac{\log \frac {1}{\delta}}{ \eta  n}~.
\]
\end{corollary}
\begin{proof}
We first observe that the generalization error can be bounded as
\begin{align*}
 \bgen &= - \frac{1}{n}\sum_{t=1}^n \iprod{P_{W_n|S_n}}{c_t} + \frac{\eta}{n}\sum_{t=1}^n 
\EEcc{\bpa{\ell(W_n,Z_t)}^2}{S_n}
\\
&= \frac{1}{n}\sum_{t=1}^n \iprod{P_t - P_{W_n|S_n}}{c_t} - \frac{1}{n}\sum_{t=1}^n \iprod{P_t}{c_t} + 
\frac{\eta}{n}\sum_{t=1}^n \EEcc{\bpa{\ell(W_n,Z_t)}^2}{S_n}\\
&\le \frac{\DDKLb{P_{W_n|S_n}}{\tP_1}}{\eta n} + \frac{3\eta}{n}\sum_{t=1}^n 
\EEcc{\bpa{\ell(W_n,Z_t)}^2}{S_n} - \frac{1}{n}\sum_{t=1}^n \iprod{P_t}{c_t},
\end{align*}
where we used the regret bound of Theorem~\ref{thm:second-order} in the last step (noting that the condition $\eta 
c_t(w) \le \frac 12$ is satisfied under our condition on $\eta$), and that
\[
 \pa{c_t(w) - g_t(w)}^2 = \pa{\ell(w,Z_t) + \eta \pa{\ell(w,Z_t)}^2}^2 \le 2 \bpa{\ell(w,Z_t)}^2
\]
holds under our conditions on $\eta$. 
To continue, we notice that the sum $-\sum_{t=1}^n \iprod{P_t}{c_t}$ is a supermartingale 
that decays at a fast rate. In particular, we rewrite this term and apply Lemma~\ref{lem:concentration_empirical} 
with $X_t = \iprod{P_t}{\loss(\cdot,Z_t)}$ and $\lambda = \eta$ to obtain the following bound that holds with 
probability at least $1-\delta$:
\begin{align*}
 -\frac{1}{n}\sum_{t=1}^n \iprod{P_t}{c_t} &= \frac{1}{n} \sum_{t=1}^n \pa{\iprod{P_t}{\EE{\loss(\cdot,Z')}} - 
\iprod{P_t}{\loss(\cdot,Z_t)} - \eta \biprod{P_t}{\pa{\loss(\cdot,Z_t)}^2}} 
\\
&\le \frac{1}{n} \sum_{t=1}^n \pa{\iprod{P_t}{\EE{\loss(\cdot,Z')}} - 
\iprod{P_t}{\loss(\cdot,Z_t)} - \eta \pa{\biprod{P_t}{\loss(\cdot,Z_t)}}^2} \le \frac{\log \frac{1}{\delta}}{\eta n},
\end{align*}
where the first inequality is Jensen's. This concludes the proof.
\end{proof}
By further upper bounding the quadratic term appearing in the upper bound by the training error, we obtain the 
following relaxation of the bound:
\begin{corollary}
Suppose that $\ell(w,z)\in[0,1]$ for all $w,z$. For any $\tP_1\in\Delta_{\Ww}$ and any $\eta \in \left[0,\frac 
12\right]$, with probability at least $1-\delta$, the generalization error of all statistical learning algorithms
$W_n=\alg(S_n)$ simultaneously satisfy 
 \[
\EEcc{\loss(W_n,Z')}{S_n} \le \frac{\frac 1n \sum_{t=1}^n \EEcc{\loss(W_n,Z_t)}{S_n}}{1-\eta} + 
\frac{\DDKL{P_{W_n|S_n}}{\tP_1}}{\eta n} + \frac{\log \frac {1}{\delta}}{2 \eta n}.
 \]
\end{corollary}
Several similar data-dependent bounds have been proposed in the PAC-Bayesian literature. For instance, the 
``PAC-Bayes-Bernstein inequality'' of \citet{SLCSA12}, which features the variance of the losses instead of their 
second moment, can be recovered from the same regret analysis as above by setting $g_t = 0$. Later results of 
\citet{TS13} and \citet{MGG19} have replaced this unobservable quantity by the empirical variance of the training loss. 
While these results were proved using sophisticated concentration inequalities combined with PAC-Bayesian ``change of 
measure'' arguments, our result directly follows from a combination of a few standard techniques from online learning. 
The main merit of these results (shared by our result above) is that they imply a fast rate of order $1/n$ when 
the training error is zero.

\subsection{Generalized PAC-Bayesian bounds via Following the Regularized Leader}\label{sec:FTRL}
We now provide a range of entirely new generalization bounds derived from a general class of online learning algorithms 
known as \emph{Follow the Regularized Leader} (FTRL, see, e.g., \citet{Rak09,SS12,Ora19}). FTRL algorithms are 
defined using a convex regularization function $h:\Delta_{\Ww}\ra\real$. For the sake of this paper, we 
concentrate on regularizers that are \emph{$\alpha$-strongly convex} with respect to a norm $\norm{\cdot}$ defined on 
the set of signed measures on $\Ww$, in the sense that the following inequality is satisfied for all 
$P,P'\in\Delta_{\Ww}$ and all $\lambda\in[0,1]$:
\begin{equation}\label{eq:strong_convexity}
 h(\lambda P + (1-\lambda) P') \le   \lambda h(P) + (1-\lambda) h(P') - \frac{\alpha \lambda (1-\lambda)}{2} \norm{P - 
P'}^2.
\end{equation}
We also assume that $h$ is proper in the sense that it never takes the value $-\infty$ and is not identically 
equal to $+\infty$, and that it is lower semicontinuous on its effective domain.
Given such a regularization function and a positive learning-rate parameter $\eta$, we can define the distribution 
$P_t$ picked by FTRL in round $t$ as
\[
 P_{t} = \argmin_{P\in \Delta_{\Ww}} \ev{\iprod{P}{\sum_{k=1}^{t-1} c_k} + \frac{1}{\eta} h(P)}.
\]
We assume throughout this section that the norm $\norm{\cdot}$ is such that the simplex $\Delta_{\Ww}$ is compact with 
respect to it, and thus under mild regularity conditions on the loss function $\ell$, the minimum exists. Furthermore,
its 
uniqueness is ensured by the strong convexity of $h$.
 In what follows, we derive generalization bounds for this vanilla version of 
FTRL as well as some of its straightforward variants, and we will instantiate the bounds with some interesting 
regularization functions.

\subsubsection{Strongly convex regularizers}
We first state a generalization bound obtained via an application of the classical FTRL analysis for strongly convex 
regularizers.
\begin{theorem}\label{thm:FTRL}
Suppose that $h$ is $\alpha$-strongly convex with respect to the norm $\norm{\cdot}$, and that there exists 
$\sigma>0$ such that $\sup_{w\in \Ww} \EXP_{Z'\sim \mu} \ell(w,Z')^2 \le \sigma^2$.
Then, for any $\eta > 
0$, with probability at least $1-\delta$,
  the generalization error of all statistical learning algorithms $W_n=\alg(S_n)$ simultaneously satisfy 
\[
 \bgen \le \frac{h(P_{W_n|S_n}) - h(P_1)}{\eta n} + \frac{\eta}{2 \alpha n} \sum_{t=1}^n 
\bnorm{\ell(\cdot,Z_t) - \EE{\ell(\cdot,Z')}}_*^2 + \sqrt{\frac{\sigma^2\log(1/\delta)}{2n}}.
\]
\end{theorem}
We provide the proof of the regret bound forming the basis of this theorem in Appendix~\ref{app:FTRL} 
(cf.~Theorem~\ref{thm:FTRL}).
A straightforward covering argument over the choice of $\eta$ gives the following bound.
\begin{corollary}\label{cor:FTRL_tuned}
Suppose that $h$ is $\alpha$-strongly convex with respect to the norm $\norm{\cdot}$. Furthermore, suppose that 
$\ell(w,z)\in[0,1]$ for all $w,z$ and that $\bnorm{\ell(\cdot,Z_t) - \EE{\ell(\cdot,Z')}}_*\le B$ for some positive 
$B$. Fix $\epsilon \in (0,1]$.
  Then, for any $P_1\in\Delta_{\Ww}$ and any $n>1$,
with probability at least $1-\delta$,
  the generalization error of all statistical learning algorithms $W_n=\alg(S_n)$ simultaneously satisfy
\[
  \bgen \le
      \left(1+\frac{\epsilon^2}{2}\right) \sqrt{\frac{B^2 \pa{h\pa{P_{W_n|S_n}} - h(P_1)}}{2\alpha n}}
  +
\sqrt{\frac{ \log \log (4\sqrt{n}) + \log(1/\delta) + \log \frac{2}{\epsilon} }{2n}}~.
\]
\end{corollary}
The proof follows the same arguments as the proof of Corollary~\ref{cor:MWA_tuned}.
The above two bounds respectively recover the PAC-Bayesian generalization bounds of Corollaries~\ref{cor:MWA_untuned} 
and~\ref{cor:MWA_tuned} when setting $h(P) = \DDKL{P}{P_1}$, which is known to be $1$-strongly convex with respect to 
the total variation norm (whose dual norm is the supremum norm). We provide further examples at the end of this section.

\subsubsection{Empirical bounds via optimistic FTRL}
One downside of the bound claimed in the previous section is that it depends on the dual norms of $\ell(\cdot,Z_t) - 
\EE{\ell(\cdot,Z')}$, which involves the unobservable quantity $\EE{\ell(\cdot,Z')}$. It is often more desirable to 
provide generalization bounds that are fully empirical and can be evaluated without having to estimate the test error 
(which would indeed defeat the purpose of proving generalization bounds in the first place). In this section, we 
provide a simple remedy to this issue by considering an \emph{optimistic} version of FTRL. Optimistic online learning 
algorithms were first proposed by \citet{RS13,RS13b} as algorithms that can take advantage of a \emph{guess} $g_t$ of 
the cost function $c_t$. Such methods provide tighter guarantees whenever $c_t$ and $g_t$ are close in an appropriate 
sense, and typically retain the worst-case guarantees of FTRL when this is not the case. In our setting, it is natural 
to pick $g_t$ as the predictable part of $c_t$ corresponding to the test loss: $g_t(w) = - \EE{\ell(w,Z')}$. Indeed, 
the only unpredictable part of the cost function from the perspective of the online learner is the empirical loss 
$\ell(w,Z_t)$ on the $t$-th data point, as the other additive term remains the same in all rounds.

We now derive a generalization bound by considering the optimistic version of FTRL that picks the following 
distribution $P_t$ in each round $t$:
\[
 P_{t} = \argmin_{P\in \Delta_{\Ww}} \ev{\iprod{P}{g_t + \sum_{k=1}^{t-1} c_k} + \frac{1}{\eta} h(P)}~,
\]
The generalization bound derived using this algorithm with the above choice of $g_t$ is stated in the following theorem.
\begin{theorem}\label{thm:FTRL_optimistic}
Suppose that $h$ is $\alpha$-strongly convex with respect to the norm $\norm{\cdot}$. Then, for any $\eta > 0$, any 
$P_1\in\Delta_{\Ww}$ and any $n>1$, with probability at least $1-\delta$,
  the generalization error of all statistical learning algorithms $W_n=\alg(S_n)$ simultaneously satisfy
\[
 \bgen \le \frac{h(P_{W_n|S_n}) - h(P_1)}{\eta n} + \frac{\eta}{2\alpha n} \sum_{t=1}^n 
 \bnorm{\ell(\cdot,Z_t)}_*^2 + \sqrt{\frac{\sigma^2\log(1/\delta)}{n}}~.
\]
\end{theorem}
The proof makes use of a standard regret bound for optimistic FTRL that we state and prove as 
Theorem~\ref{thm:optimistic_FTRL_regret} in Appendix~\ref{app:FTRL} for completeness.
As can be seen from the statement, this bound improves over that of Theorem~\ref{thm:FTRL} in that it replaces the norm 
of the generalization gap with the norm of the training loss, which can be empirically measured. An optimally tuned 
version similar to Theorem~\ref{cor:FTRL_tuned} can be derived using similar techniques as before.

\subsubsection{Example: $p$-norm regularization}
From the perspective of convex analysis, the family of $p$-norm distances is a natural candidate for defining 
regularization functions. Concretely, we define the weighted $p$-norm distance between the signed measures 
$P,P'\in\Delta_{\Ww}$ and base measure $P_1$ as the $L_p$ distance between their Radon--Nykodim derivatives with 
respect to $P_1$:
\begin{equation}\label{eq:pnorm}
 \norm{P - P'}_{p,P_1} = \pa{\int_{\Ww} \pa{\frac{\dd P}{\dd P_1} - \frac{\dd P'}{\dd P_1}}^p \dd P_1}^{1/p}~.
\end{equation}
The corresponding dual norm is the $L_q$-norm defined for all $f$ as
\[
 \norm{f}_{q,P_1,*} = \pa{\int_{\Ww} f^q \dd P_1}^{1/q}~,
\]
with $q > 1$ such that $1/p + 1/q = 1$. 
It is useful to note that the distance $\norm{P - P_1}_{p,P_1}^p$ is {an instance of the family of so-called 
$f$-divergences \citep{renyi1961measures,csiszar1963informationstheoretische}, with generator function $\varphi(x) = 
(x-1)^p$.} These divergences are alternatively known under a variety of names such as Hellinger divergence of 
order $p$, $p$-Tsallis divergence or simply $\alpha$-divergence with $\alpha = p$ (see, e.g., \citet{SV16,NN11}). The 
case $p=2$ is often given special attention, and the corresponding squared norm can be seen to match Pearson's 
$\chi^2$-divergence \citep{Pea00}. We denote this divergence by $\mathcal{D}_{\chi^2}$ below.

Powers of the norm defined above exhibit different strong-convexity properties depending on the value of $p$, with two 
distinct regimes $p \in (1,2]$ and $p>2$. 
The following corollary summarizes the results obtained in these two regimes when respectively setting $h(P) = \norm{P 
- P_1}_{p,P_1}^2$ and $h(P) = \norm{P - P_1}_{p,P_1}^p$:
\begin{corollary}
\label{cor:pnorm}
Fix $p > 1$ and $q$ such that $1/p + 1/q = 1$, and suppose that there exists $B > 0$ such that $\sup_{w\in \Ww} 
\EXP_{Z'\sim \mu} \pa{\ell(w,Z')}^{\min\ev{2,q}} \le B$. Then, for any $P_1\in\Delta_{\Ww}$ and any $n>1$, with 
probability at least $1-\delta$, the generalization error of all statistical learning algorithms $W_n=\alg(S_n)$ 
simultaneously satisfy the following bounds:
\begin{enumerate}[label=(\alph*)]
 \item For $p\in(1,2]$, 
 \[
  \bgen \le \frac{\norm{P_{W_n|S_n} - P_1}^2_{p,P_1}}{\eta n} + 
\frac{\eta}{(p-1)n} \sum_{t=1}^n\bnorm{\loss(\cdot,Z_t)}_{q,P_1,*}^2 + \sqrt{\frac{B^2\log(1/\delta)}{2n}}~.
 \]
 \item For $p\ge 2$ and $q$, 
 \[
  \bgen \le \frac{\norm{P_{W_n|S_n} - P_1}_{p,P_1}}{\eta n} + 
\frac{\eta^{q-1}}{2^{q-1}qn}\sum_{t=1}^n\bnorm{\loss(\cdot,Z_t)}^q_{q,P_1,*}  + 
B\pa{\frac{\log(1/\delta)}{n}}^{1-1/q}~.
 \]
\end{enumerate}
\end{corollary}
These bounds are fully empirical in the sense that they depend on the realization of the data sequence, up to an 
additional martingale concentration term of order $B\pa{\log(1/\delta)/n}^{1-1/q}$. The perhaps unexpectedly 
mild dependence on the confidence level $\delta$ is due to the fact that the martingale average $-M_{\Pi_n}$ that needs 
to be probabilistically controlled is the \emph{lower tail} of a sum of nonnegative random variables, which can be
effectively bounded even for heavy-tailed random variables. Upper-bounding the data-dependent terms using Markov's 
inequality, one can recover the results of \citet{BGLR16} and \citet{AG18} that were proved under the much stronger 
assumptions that the losses are bounded or that they always have finite variance. {Moreover, our results 
massively improve the scaling with respect to the confidence parameter $\delta$ in comparison to these works, going 
from a polynomial dependence on $1/\delta$ to a polylogarithmic one.}
\citet{RBTS21} derive a comparable 
bound on the \emph{expected} generalization error in the special case $p=2$. 
Our bounds in the regime $p>2$ do not require such assumptions and remain meaningful when the losses are heavy 
tailed and the $q$-th moment of the random loss is bounded only for some $q < 2$. In such cases, our result implies a 
slow rate of $n^{-(1-1/q)}$ for the generalization error, which is expected when dealing with concentration of 
heavy-tailed random variables \citep{gnedenko1954limit}. In the regime $p\in(1,2]$, our 
bound interpolates between the guarantee for $p=2$ and the standard PAC-Bayesian bound of 
Corollary~\ref{cor:MWA_untuned} as $p$ approaches $1$, at least in terms of dependence on the $L_q$-norm of the loss 
function. In terms of dependence on the divergence measures, this interpolation fails as $p$ tends to $1$, as the 
squared $L_p$-divergence converges to the squared total variation distance which is not strongly convex. Accordingly, 
the bound blows up in this regime and Corollary~\ref{cor:MWA_untuned} gives a strictly better bound. All of these 
guarantees require the boundedness of $\norm{P_n - P_1}_{p,P_1}$, which becomes a more and more stringent condition as 
$p$ increases.

The results in Corollary~\ref{cor:pnorm} are direct consequences of Theorem~\ref{thm:FTRL}. The case $p=2$ is 
the simplest and can be proved by picking $h(P) = \DDchi{P}{P_1}$. Being a squared $2$-norm, $h$ is obviously 
$1$-strongly convex with respect to $\norm{P - P_0}_{2,P_0}$ as it satisfies the condition of 
Equation~\eqref{eq:strong_convexity} with equality. A similar argument works for the regime 
$p\in(1,2]$, where the choice $h(P) = \norm{P - P'}_{p,P_0}^2$ exhibits $2(p-1)$-strong convexity with respect to the 
norm $ \norm{\cdot}_{p,P_0}$ (see, e.g., Proposition~3 in \citet{BCL94}, that also establishes that strong convexity 
does not hold for $p > 2$). The case $p\ge 2$ is slightly more complex and it requires minor adjustments to the proof 
of Theorem~\ref{thm:FTRL}. We defer the proof of this case to Appendix~\ref{app:pnorm}.

\subsubsection{Example: smoothed relative-entropy regularization}\label{sec:wasserstein}
Let us now consider the special case where the hypothesis space is $\Ww = \real^d$. In this case, a common idea 
(used by, e.g. \citep{LC01,DR17}) in the 
PAC-Bayesian literature is to smooth the posterior distribution $P_{W_n|S_n}$ by adding Gaussian noise to an otherwise 
deterministic output hypothesis $W_n^*$ to ensure the boundedness of $\DDKL{P_{W_n|S_n}}{P_1}$. The effect of this 
perturbation is then typically addressed by analyzing the gap $\EEcc{\ell(W_n,Z') - \ell(W_n^*,Z')}{S_n}$. Making sure 
that this gap does not end up dominating the bound generally necessitates using perturbation levels that go to zero for 
large $n$. In this section, we provide an alternative smoothing method that allows using constant perturbation levels 
for a class of smooth functions.

Our approach is based on an FTRL variant based on a smoothed version of the relative entropy as regularization 
function. 
In order to construct this regularizer, we define the Gaussian smoothing operator $G_\sm$ that acts on any distribution 
$P\in\Delta_{\Ww}$ as 
$G_\sm P = \int_{\Ww} \mathcal{N}(w, \sm^2 I) \dd P(w)$, where $\mathcal{N}(w,\sm^2 I)$ is the $d$-dimensional 
Gaussian distribution with mean $w$ and covariance $\sm^2 I$. Using this operator, we define the $\sm$-smoothed 
relative entropy as $\DDsigma{P}{P'} = \DDKL{G_\sm P}{G_\sm P'}$ and set $h(P) = \DDsigma{P}{P_1}$. Similarly, we 
define the smoothed total variation distance between $P$ and $P'$ as $ \norm{P - P'}_\sm = \norm{G_\sm P - 
G_\sm P'}_{\mathrm{TV}}$. Both of these divergences have the attractive property that they remain meaningfully bounded 
under much milder assumptions than their unsmoothed counterparts, even when the supports of $P$ and $P'$ are disjoint.

It is straightforward to verify\footnote{For instance, this can be shown by studying the Bregman divergence 
associated with $h$, which is easily seen to satisfy
$
\mathcal{B}_h\bigl(P\bigm\|P'\bigr)   = \mathcal{D}_\gamma\bigl(P\bigm\|P'\bigr) \ge \frac 12 \bnorm{G_\sm\bpa{P - 
P'}}^2_{\mathrm{TV}} = \frac 12 \bnorm{P - P'}_\sm^2$, which in turn implies strong convexity of $h$. We refer to 
Appendix~\ref{app:FTRL} for more details on techniques for handling strongly convex regularization functions.} that $h$ 
is $1$-strong convexity in terms of the smoothed total variation distance $\norm{\cdot}_{\sm}$. The dual norm of the 
smoothed TV 
distance is defined as $\norm{f}_{\sm,*} = \sup_{\norm{P-P'}_\sm \le 1} \iprod{f}{P-P'}$, which, together with 
the above arguments and Theorem~\ref{thm:FTRL_optimistic}, immediately implies the following result:
\begin{corollary}\label{cor:smoothed_plain}
Suppose that there exists $\sigma>0$ such that $\sup_{w\in \Ww} \EXP_{Z'\sim \mu} \ell(w,Z')^2 \le \sigma^2$.
Then, for any $\eta > 0$, any $\gamma \ge 0$, any $P_1\in\Delta_{\Ww}$ and any $n>1$, with probability at least 
$1-\delta$,
  the generalization error of all statistical learning algorithms $W_n=\alg(S_n)$ simultaneously satisfy
 \[
\bgen \le \frac{\DDsigma{P_{W_n|S_n}}{P_0}}{\eta n} + \frac{\eta}{n} \sum_{t=1}^n 
\norm{\loss(\cdot,Z_t)}^2_{\sm,*} + \sqrt{\frac{2\sigma^2 \log \frac{1}{\delta}}{n}}~. 
\]
\end{corollary}
It remains to be shown that the dual norm $\norm{\ell(\cdot,z)}_{\sm,*}$ can be bounded meaningfully.
By the intuitive properties of the smoothed total variation distance, one can reasonably expect this norm to capture 
the 
smoothness properties of the loss function, and that it is small whenever $\ell(\cdot,z)$ is bounded and highly smooth. 
{In what follows, we provide an upper bound on this norm that holds for a class of infinitely smooth functions 
with bounded directional derivatives. To make this definition precise, we let $B_1 = \ev{v\in\real^d: \twonorm{v} = 1}$ 
denote the Euclidean unit ball, define $D^0 f = f$, and for each subsequent $j>0$, we define the $j$-th
directional derivative function $D^j f:\Ww \times B_1^j$ via the assignment
\[
D^j f(w|v_1,v_2,\dots,v_j) = \lim_{c\ra 0} \frac{D^{j-1} f(w + c v_j|v_1,v_2,\dots,v_{j-1}) - D^{j-1} 
f(w|v_1,v_2,\dots,v_{j-1})}{c}.
\]
Notice that $D^j$ is linear in $f$.
}
We say that a function $f$ is infinitely smooth if all of its higher-order directional derivatives exist and satisfy 
$D^j f(w|v_1,v_2,\dots,v_j) \le \beta_j$ for all directions $v_1,v_2,\dots,v_j$, all $w\in\Ww$, and all $j$. 
For such functions, the following lemma provides an upper bound on $\norm{f}_{\sm,*}$:
\begin{lemma}\label{lem:smoothing}
Suppose that $f$ is infinitely smooth in the above sense.
Then, the dual norm $\norm{f}_{\sm,*}$ satisfies $\norm{f}_{\sm,*} \le \sum_{j=0}^\infty \bpa{\sm \sqrt{d}}^j 
\beta_j$.
\end{lemma}
The proof is based on a successive smoothing argument and is provided in Appendix~\ref{app:smoothing}.
With the help of this lemma, we can thus obtain the following result:
\begin{corollary}\label{cor:smoothed_PAC_bayes}
Suppose that $\ell(\cdot,z)$ is infinitely smooth for all $z$ with $\beta_j \le \beta$ for all $j\ge 0$, that there 
exists $\sigma>0$ such that $\sup_{w\in \Ww} \EXP_{Z'\sim \mu} \ell(w,Z')^2 \le \sigma^2$, and that 
$\sm < 1/\sqrt{d}$. Then, for any $\eta > 0$, any $P_1\in\Delta_{\Ww}$ and any $n>1$, with probability at least 
$1-\delta$,  the generalization error of all statistical learning algorithms $W_n=\alg(S_n)$ simultaneously satisfy
 \[
\bgen \le \frac{\DDsigma{P_{W_n|S_n}}{P_0}}{\eta n} + \frac{\eta \beta}{1 - \sm \sqrt{d}} + 
\sqrt{\frac{2\sigma^2 \log \frac{1}{\delta}}{n}}~. 
\]
\end{corollary}
The result follows from applying Corollary~\ref{cor:smoothed_plain} and upper-bounding 
\[
\norm{\loss(\cdot,Z_t)}^2_{\sm,*} \le \beta \sum_{j=0}^\infty \bpa{\sm \sqrt{d}}^j = \frac{\beta}{1-\sm \sqrt{d}}.
\]
Setting the smoothing level as $\sm = 1/(2\sqrt{d})$ guarantees that the second term is of order $\eta \beta$. 
To our knowledge, PAC-Bayesian guarantees of similar form can only be derived for perturbation levels that decay to 
zero as $n$ increases, which severely limits the potential gains that can be achieved by smoothing.

In order to obtain a more explicit bound, we note that the smoothed relative entropy can be upper-bounded in terms of 
the squared Wasserstein-2 distance as $\DDsigma{P}{P'} \le \frac{1}{2\sm^2} \mathbb{W}_2^2(P,P')$. For completeness, 
we give the precise definition of the Wasserstein distance $\mathbb{W}_2$ and a direct proof of this result in 
Appendix~\ref{app:wasserstein}. The following corollary states the resulting bound when setting $\sm = 1/(2\sqrt{d})$.
\begin{corollary}\label{cor:wasserstein}
Suppose that $\ell(\cdot,z)$ is infinitely smooth for all $z$ with $\beta_j \le \beta$ for all $j\ge 0$, and 
that there exists $\sigma>0$ such that $\sup_{w\in \Ww} \EXP_{Z'\sim \mu} \ell(w,Z')^2 \le \sigma^2$. Then, for 
any $\eta > 0$, any $P_1\in\Delta_{\Ww}$ and any $n>1$, with probability at least 
$1-\delta$,  the generalization error of all statistical learning algorithms $W_n=\alg(S_n)$ simultaneously satisfy
 \[
\bgen \le \frac{2 d \mathbb{W}_2^2\pa{P_{n|S},P_1}}{\eta n} + 2\beta\eta + 
\sqrt{\frac{\sigma^2\log(1/\delta)}{n}}~.
\]
\end{corollary}
We are not aware of any directly comparable results in the literature. \citet{ZLT18}, \citet{WDSC19} and \citet{RBTS21} 
provide vaguely similar guarantees that depend on the Wasserstein-1 distance and only require bounded first 
derivatives, but it is not clear if these bounds are decreasing with the sample size $n$ in general. Whenever all 
hypotheses satisfy $\twonorm{w}\le R$ for some $R$, the result stated above implies an upper bound on the expected 
generalization error that scales as $R\sqrt{\beta d / n}$, which is directly comparable with what one might obtain via 
a 
straightforward uniform convergence argument 
involving the covering number of Lipschitz functions on a bounded domain (see, e.g., \citet{Dud84}). The dependence 
on 
the dimension $d$ of such guarantees can be relaxed or completely removed when assuming more structure about the loss 
function \citep{Bar98,WSS00,Zha02}. Whether such arguments can be applied to remove the dependence on $d$ from the 
above 
bound is a curious problem we leave open for future research.

  \section{Extensions}\label{sec:extensions}
The online-to-PAC conversion framework naturally lends itself to a
number of extensions that allow one to prove more and 
more advanced generalization bounds. In this section, we provide a few such extensions and briefly explain how they can 
be applied to tighten the bounds presented in the previous sections.

\subsection{Conditional online-to-PAC conversions}\label{sec:conditional}
One limitation of standard PAC-Bayesian bounds is that the prior $P_1$ is not allowed to depend in any way on the 
training data $S_n$. This entails several further limitations, for instance that the bounds can become vacuous even in 
the simplest setting of learning one-dimensional threshold classifiers (see, e.g., \citet{LM20}). 
Our framework presented in Section~\ref{sec:o2PAC} shares the same limitation: {the online learner is only allowed 
sequential access to the training data, and in particular, its predictions in round $t$ can only depend on data points 
$Z_1,\dots, Z_{t-1}$}. As such, the above limitations of PAC-Bayes are also inherited. 
In the context of PAC-Bayesian and information-theoretic generalization, this issue has been successfully addressed 
by \citet{Aud04} and \citet{Cat07} via the notion of ``almost exchangeable priors'', which allows the recovery of all 
classical PAC-learning bounds from a PAC-Bayesian framework. This idea has recently been 
rediscovered and popularized by the works of \citet{SZ20} and \citet{HD20a} (see also \citet{HDMR21,GSZ21}). 

As we show below, our framework can readily address the issue at hand via a simple extension that we call 
\emph{conditional online-to-PAC conversion}. Similarly to the frameworks described above, we define a 
supersample of $2n$ data points sampled i.i.d.~from the distribution $\mu$, denoted as $\wh{S}_{n} = 
(Z_1^{+1},Z_1^{-1}, Z_2^{+1}, Z_2^{-1},\dots,Z_n^{+1}, Z_n^{-1})$. Based on these samples, we define the
\emph{conditional generalization game} where the following steps are repeated in a sequence of rounds $t=1,2,\dots,n$:
\begin{enumerate}
 \item The online learner picks a distribution $P_t \in \Delta_{\Ww}$;
 \item the environment draws an index $I_t \in \ev{-1,+1}$ uniformly at random, and selects the cost function $c_t: 
\Ww\ra \real$ defined for each $w\in\Ww$ as
 \[
  c_t(w) = \ell(w,Z_t^{I_t}) - \frac{1}{2} \pa{\ell(w,Z_t^{+1}) - \ell(w,Z_t^{-1})}~;
 \]
 \item the online learner incurs cost $\iprod{P_t}{c_t}$;
 \item the environment reveals the index $I_t$ to the online learner.
\end{enumerate}
In this setup, the online learner is allowed even more knowledge than in the standard generalization game: besides 
knowing the loss function $\ell$ and the data distribution $\mu$, the online learner is also allowed to know the entire 
supersample $\wh{S}_n$. Thus, revealing the index $I_t$ in round $t$ reveals the entire cost function $c_t$ to the 
learner. The outcome of the game is associated with the training and test loss of the statistical learning algorithm by 
setting the training set as $S_n = \bev{Z_t^{I_t}}_{t=1}^n$ and the test set as $S'_n = \bev{Z_t^{-I_t}}_{t=1}^n$. 
{Accordingly, we will use the notation $Z_t = Z_t^{I_t}$ and $Z_t' = Z_t^{I_t}$ in parts of the discussion below.}

We treat the additional formalities by introducing the following notation. We use $\wh{\F}_{t}$ to denote the 
sigma-algebra induced by the \emph{entire data set $\wh{S}_n$}, and the sequence of random variables generated and used 
by both players (including the random coin flips $I_1,\dots,I_t$ and all potential randomization utilized by the online 
learner) up until the end of round $t$. 
In this protocol, an online learning algorithm 
 $\Pi_n = \{P_t\}_{t=1}^n$ is a sequence of 
functions such that $P_t$ maps the sequence of past
  coin flips $(i_1,\ldots,i_{t-1})\in \ev{-1,1}^{t-1}$ \emph{and data sets $\wh{s}_n\in\Zw^{2n}$} to $\Delta_{\Ww}$, the
  set of all probability distributions over the hypothesis class
  $\Ww$. For the brevity of notation, we abbreviate  $P_t=P_t(I_1,\ldots,I_{t-1},\wh{S}_n)$, and we denote by 
$\wh{\Pw}_n$ the 
class of all   online learning algorithms over sequences of length $n$ in this protocol.\looseness=-1

In the conditional generalization game defined above, our technique gives an upper bound on the \emph{empirical 
generalization error} defined as
\[
 \wh{\gen}(W_n,\wh{S}_n) = \frac 1n \sum_{t=1}^n \pa{\ell(W_n,Z_t') - \ell(W_n,Z_t)}.
\]
In order to bound the generalization error, we need to control the gap 
between the two quantities, denoted as
\[
 \Delta(W_n,\wh{S}_n)= \gen(W_n,S_n) - \wh{\gen}(W_n,\wh{S}_n) = \frac 1n \sum_{t=1}^n \pa{\EEcc{\ell(W_n,Z')}{W_n} - 
\ell(W_n,Z_t')}.
\]
  It is important to note that $W_n$ only depends on
$S_n=(Z_1,\ldots,Z_n)$ and therefore, conditioned on $S_n$, the
quantity $\Delta(W_n,\wh{S}_n)$ is the difference between the empirical mean of nonnegative
 i.i.d.~random variables and its expectation. Hence, the lower tail of $\Delta(W_n,\wh{S}_n)$
 can be controlled using  standard tools.
 As before, the total cost accumulated by the online learner is denoted by $M_{\Pi_n}= \frac 
1n \sum_{t=1}^n \iprod{P_t}{c_t}$. Importantly, this quantity is a
normalized sum of martingale differences, when conditioned on $\wh{S}_n$, due to the conditional independence of 
$P_t$ 
and $c_t$:
\begin{align*}
 \mathbb{E}\bigl[\iprod{P_t}{c_t}\big|\F_{t-1},\wh{S}_n\bigr] &= \mathbb{E}_{\wt{W}_t\sim 
P_t}\bigl[c_t\bpa{\wt{W}_t}\big| \F_{t-1},\wh{S}_n\bigr]
 \\
 &= \mathbb{E}_{\wt{W}_t\sim P_t}\Bigl[\loss\bpa{\wt{W}_t,Z_t^{I_t}} - \frac 12 \Bpa{\loss\bpa{\wt{W}_t,Z_t^{+1}} 
+ \loss\bpa{\wt{W}_t,Z_t^{-1}}}\Big| \F_{t-1},\wh{S}_n\Bigr] = 0~.
\end{align*}
Indeed, the online learning protocol guarantees that $\wt{W}_t$ is chosen before the index $I_t$ is revealed to the 
online learner, and thus the conditional expectation of $\loss\bpa{\wt{W}_t,Z_t^{I_t}}$ is 
$\frac 12\bpa{\loss\bpa{\wt{W}_t,Z_t^{+1}} +
  \loss\bpa{\wt{W}_t,Z_t^{-1}}}$.
Note that $M_{\Pi_n}$ is \emph{not} a 
martingale average without conditioning on $\wh{S}_n$, as each $W_t$ may depend on the entire data set, including 
future 
instances $Z_s^{\pm 1}$ with $s > t$.

The following result connects the generalization error to the regret in a similar way as Theorem~\ref{thm:main} does 
for 
the basic generalization game:
\begin{theorem}\label{thm:conditional}
The generalization error of any learning algorithm $W_n = \alg(S_n)$
satisfies that, for any online learning algorithm $\Pi_n \in \wh{\Pw}_n$,
\[
 \bgen = \frac{\regret_{\Pi_n}(P_{W_n|S_n})}{n} - M_{\Pi_n} + \EEcc{\Delta(W_n,\wh{S}_n)}{\wh{S}_n} ~.
\]
In particular, for any class $\C_n \subset \wh{\Pw}_n$ of online learning algorithms,
\[
 \bgen = \inf_{\Pi_n \in \C_n} \left(\frac{\regret_{\Pi_n}(P_{W_n|S_n})}{n} 
- M_{\Pi_n} \right) + \EEcc{\Delta(W_n,\wh{S}_n)}{\wh{S}_n}~.
\]
\end{theorem}
\begin{proof}
For simplicity, we denote the training points as $Z_t = Z_t^{I_t}$ and the test points as $Z_t' = Z_t^{-I_t}$.
Recalling the notation $\iprod{P}{f} = \EEs{f(W)}{W\sim P}$, we rewrite the conditional expectation of the 
empirical generalization error as follows:
\begin{align*}
\EEcc{\wh{\gen}(W_n,\wh{S}_n)}{\wh{S}_n} &= \frac 1n \sum_{t=1}^n {\EEcc{\ell(W_n,Z_t') - \ell(W_n,Z_t)}{\wh{S}_n}  
}
= - \frac 1n \sum_{t=1}^n \EEcc{c_t(W_n)}{\wh{S}_n}
\\
&= - \frac 1n \sum_{t=1}^n \iprod{P_{W_n|\wh{S}_n}}{c_t} = - \frac 1n \sum_{t=1}^n \iprod{P_{W_n|S_n}}{c_t}  
\\
&= \frac 1n \sum_{t=1}^n \iprod{P_t - P_{W_n|S_n}}{c_t}
- \frac 1n \sum_{t=1}^n \iprod{P_t }{c_t}
\\
&= \frac {\regret_n(P_{W_n|S_n})}{n} - \frac 1n \sum_{t=1}^n \iprod{P_t}{c_t},
\end{align*}
where in the second line we observed that $P_{W_n|\wh{S}_n} = P_{W_n|S_n}$ holds due to the construction of $W_n$ that 
only depends on the training data $S_n$.
Recalling the definitions of $M_{\Pi_n}$ and $\Delta(W_n,\wh{S}_n)$ completes the proof.
\end{proof}
The following corollary instantiates the bound for a single algorithm under the assumption that the loss function is 
bounded in $[0,1]$.
\begin{corollary}
\label{cor:single_conditional}
Consider an arbitrary online learning algorithm $\Pi_n \in \wh{\Pw}_n$ and suppose 
  that $\ell(w,z)\in[0,1]$ for all $w,z$.
  Then, with probability at least $1-\delta$,
  the generalization error of all statistical learning algorithms $W_n=\alg(S_n)$ simultaneously satisfy
\[
 \bgen \le
 \frac{\regret_{\Pi_n}(P_{W_n|S_n})}{n}
 + \sqrt{\frac{2\log\pa{\frac{1}{\delta}}}{n}}~.
\]
\end{corollary}
The bound follows from applying the Hoeffding--Azuma inequality for the martingale averages $M_{\Pi_n}$ and for
$\EEcc{\Delta(W_n,\wh{S}_n)}{\wh{S}_n}$. In principle, the latter term can be bounded under 
the milder condition that the second moment of the loss function is bounded for all $w$ (due to it being a lower tail 
of 
a nonnegative martingale), but boundedness of the loss function is
needed now  to ensure that the martingale 
increments constituting $M_{\Pi_n}$ are bounded almost surely. Indeed, notice that $M_{\Pi_n}$ is a martingale 
\emph{only when conditioned on the entire supersample $\wh{S}_n$}, and thus it needs to be controlled for all 
realizations of the data sequence.

The advantage of the conditional online-to-PAC conversion framework is that it allows the online learner to have prior 
knowledge of the supersample $\wh{S}_n$, which includes the training data. In particular, the online learner can now 
use data-dependent learning rates and regularization functions. 
To illustrate the use of this framework, we provide a simple application that leads to a conditional PAC-Bayesian 
generalization bound. In particular, using the standard multiplicative weights algorithm for the online learner leads 
to the following generalization bound:
\begin{corollary}\label{cor:MWA_untuned_conditional}
Suppose that $\ell(w,z)\in[0,1]$ for all $w,z$ and let $\wh{S}_n$ be a set of $2n$ i.i.d.~data points drawn from $\mu$. 
Then, for any $\eta > 0$ and any $P_1\in\Delta_{\Ww}$ that potentially depend on $\wh{S}_n$, and any $n>1$, 
with probability at least $1-\delta$,
  the generalization error of all statistical learning algorithms $W_n=\alg(S_n)$ simultaneously satisfy
\[
 \EEcc{\gen(W_n,S_n)}{S_n} \le \frac{\DDKL{P_{W_n|S_n}}{P_1}}{\eta n} + \frac{\eta}{8} + 
\sqrt{\frac{2\log(1/\delta)}{n}}~.
\]
\end{corollary}
The power of this result resides in the fact that the learning rate $\eta$ and the prior $P_1$ are allowed to depend on 
the supersample $\wh{S}_n$. Such priors are called \emph{almost exchangeable} by \citet{Aud04}, who proved an 
analogous result using classic PAC-Bayesian methodology. With a special choice of prior, and relaxing the 
high-probability bound above to only hold on expectation, this result also recovers the conditional 
information-theoretic bound of \citet{SZ20} (see also \citet{GSZ21}).

More generally, we can obtain similar conditional versions of \emph{all} generalization bounds derived in earlier 
sections of this work, including the data-dependent bounds of Section~\ref{sec:second-order} and the parameter-free 
bound of Section~\ref{sec:parameter-free}. Furthermore, the conditional generalization game allows the online learner 
to use a conditional version of FTRL, where $\eta$ and $h$ can both depend on the supersample $\wh{S}_n$, for instance 
by setting $h$ as a convex divergence measure between $P$ and a data-dependent prior $P_1$.

\subsection{Bounds on the expected generalization error}
Besides the high-confidence guarantees provided in the rest of this paper, it is straightforward to derive bounds on 
the expected generalization error using our framework. Such relaxations of PAC-Bayesian guarantees have been 
extensively studied in the last few years under the moniker ``information-theoretic generalization bounds'' 
\citep{RZ16,RZ19,XR17}. Here, we derive a generalized version of the bounds proposed in these works using our 
online-to-PAC conversion scheme. In particular, the following guarantee can be deduced directly
from Theorem~\ref{thm:main}:
\begin{corollary}\label{cor:expectation}
 Consider any statistical learning algorithm $W_n = \alg(S_n)$ and a class $\C_n \subset \Pw_n$ of online learning 
algorithms. Then, the generalization error of $\alg$ satisfies
\[
\EE{\gen(W_n,S_n)} = \inf_{\Pi_n\in \C_n} \frac{\EE{\regret_{\Pi_n}(P_{W_n|S_n})}}{n}.
\]
\end{corollary}
While this bound is decidedly weaker than its previously discussed versions, it has some merits that make it worthy of 
consideration. Most notably, since the right-hand side features an infimum over \emph{deterministic} quantities, one 
can directly obtain generalization bounds without requiring any sophisticated techniques to control suprema of 
empirical processes. We illustrate this benefit below, as well as explain the connections with the related literature.

As a simple application of Corollary~\ref{cor:expectation}, the following analogue of Corollary~\ref{cor:MWA_untuned} 
can be easily derived by choosing the multiplicative weights algorithm as the online learning method:
\begin{corollary}\label{cor:MWA_untuned_expectation}
Suppose that $\ell(w,z)\in[0,1]$ for all $w,z$. Then, for any fixed $\eta > 0$, 
any $P_1\in\Delta_{\Ww}$ and any $n>1$, the expected generalization error of any statistical learning algorithm 
$W_n = \alg(S_n)$ satisfies the 
bound
\[
 \EE{\gen(W_n,S_n)} \le \frac{\EE{\DDKL{P_{W_n|S_n}}{P_{1}}}}{\eta n} + \frac{\eta}{8}~.
\]
In particular, letting $P_{W_n}$ denote the marginal distribution of
$W_n$, we have
\[
 \EE{\gen(W_n,S_n)} \le \inf_{P_1\in\Delta_{\Ww}} \sqrt{\frac{\EE{\DDKL{P_{W_n|S_n}}{P_1}}}{2 
n}} = \sqrt{\frac{\EE{\DDKL{P_{W_n|S_n}}{P_{W_n}}}}{2 n}}~.
\]
\end{corollary}
The divergence $\EE{\DDKL{P_{W_n|S_n}}{P_{W_n}}}$ appearing in the second bound is the mutual information between $W_n$ 
and $S_n$, and thus this result recovers the bounds of \citet{XR17} that are stated in terms of the same quantity. 
Our result can be verified directly after making the following simple observations. First, notice that 
the value of $\eta$ minimizing the first bound for a given $P_1$ is non-random, and thus we can avoid the covering 
argument required for the proof of Corollary~\ref{cor:MWA_untuned}. Second, by the variational characterization of the 
mutual information, we have $\inf_{P_1\in\Delta_{\Ww}} \EE{\DDKL{P_{W_n|S_n}}{P_{1}}} = 
\EE{\DDKL{P_{W_n|S_n}}{P_{W_n}}}$ \citep{Kem74}.

All other guarantees stated in earlier sections can be adjusted analogously. Most importantly, all results in the 
preliminary version of this work \citep{LN22} can be exactly recovered by adapting the results in 
Section~\ref{sec:FTRL}, using FTRL as the online learning algorithm. ``Conditional'' analogues to the same guarantees 
can be derived via the construction proposed in Section~\ref{sec:conditional}, and in particular the ``conditional 
mutual-information'' bounds of \citet{SZ20} can be recovered via the same argument as we used above for 
Corollary~\ref{cor:MWA_untuned_expectation}.

\section{Conclusion}\label{sec:conc}
Our new online-to-PAC conversion scheme establishes a link between online and statistical learning that
provides a flexible framework for proving generalization bounds using regret analysis. In the present paper, we 
provide a short list of applications of this technique to derive generalization bounds, recovering several 
state-of-the-art results and improving them in several minor ways. These results are most likely only scratching the 
surface of what this framework is able to achieve, and in fact we feel that we have opened more questions in this work 
than what we have managed to answer. We discuss some of the numerous exciting directions for future work below.

In recent years, several connections have emerged between regret analysis in online learning, generalization bounds, 
and concentration inequalities. Early forerunners of such results are \citet{Zha02} and \citet{KST08} who 
respectively used online learning techniques to bound covering numbers and Rademacher complexities of linear function 
classes, both well-studied proxies of the generalization error in statistical learning. Some years later, a sequence 
of works by \citet{RS17} and \citep{FRS15,FRS17,FRS18} established a deep connection between uniform 
convergence of collections of martingales and the existence of online algorithms with bounded regret, effectively 
showing that any worst-case regret bound for a class of online learning games can be turned into a martingale 
concentration inequality and vice versa. Our result can be seen as a variant of the above results that is 
more tightly adapted to analyzing the generalization error of statistical learning algorithms, our key idea being a 
more specific choice of comparator point in the definition of regret. A potentially closely related line of work 
proposes to derive concentration bounds on the means of random variables using sequential betting strategies 
\citep{WSR20,OJ21}. This can be seen as a reduction to another online learning problem with the logarithmic loss 
function (cf.~Chapter 9 of \citet{CBLu06:book}). This is to be contrasted with the linear loss functions used in our 
work and all of the previously mentioned ones, and we wonder if a closer connection could be made between these 
approaches.

Our techniques are quite different from those that have been traditionally used for proving generalization bounds. 
Instead of combinatorial arguments used for studying suprema of empirical processes, our results make use of regret 
analysis, which itself traditionally builds on tools from convex analysis and optimization. This strikes us as an 
entirely new approach to study this fundamental problem of statistical learning, and also as an unexpected new 
application of convex analysis and optimization that may open interesting research directions in both of these fields.
Indeed,  much of the online learning and convex optimization literature is focused on questions of
computational efficiency that are entirely absent in our setup: since we only need to prove the \emph{existence} of 
online learning algorithms with appropriate regret guarantees, we never have to worry about implementation issues. We 
believe that this aspect can open up a new and interesting direction of research not only within statistical learning 
theory, but more broadly in convex optimization. 

Our analysis framework appears to be flexible enough to go beyond the standard statistical learning 
framework that assumes i.i.d.~data. The fact that the key part of our bounds are 
controlled almost surely via regret analysis is encouraging in that it suggests that at least some probabilistic 
assumptions can be dropped, but this still leaves us with designing appropriate notions of generalization for 
non-i.i.d.~data. It seems straightforward (and natural) to generalize our results to 
stationary data sequences by adjusting the definition of the test error $\EE{\ell(w,Z')}$ to involve an expectation 
taken with respect to the stationary distribution of $Z_t'$. We wonder if it is possible to go more significantly 
beyond the standard model by dropping even more probabilistic assumptions on the data sequence, and adapt our framework 
to deal with problems of ``out-of-distribution'' generalization.

Finally, we note that we expect that our framework will be able to capture several more concepts used in the 
statistical learning theory literature to explain generalization. Such ideas include stability 
\citep{devroye1979distribution,bousquet2002stability,mukherjee2006learning,shalev2010learnability,HRS16},
differential privacy \citep{D+06,D+06b,CMS11,BST14}, or various margin and noise conditions 
\citep{BaBoLu01,BaMe06,vEGMRW15}. We leave such extensions open for future work.

\begin{acks}
  The present paper is a significant expansion of our earlier work published at COLT 2022 \citep{LN22}. In this 
previous work, we have proved only a small subset of the results in the present paper, using a considerably more 
complicated analysis that only yielded bounds that hold in expectation. Indeed, the analysis in \citet{LN22} relied on 
direct convex-analytic calculations inspired by the analysis of a specific online learning algorithm (FTRL), which is 
now only one of the many applications captured by our general framework (cf.~Section~\ref{sec:FTRL}). We have managed 
to significantly generalize these earlier results after learning about the work of \citet*{KST08} that several 
colleagues have brought to our attention after hearing about our work at COLT 2022. We are grateful to the 
COLT community for helping us make this connection.

The authors wish to thank the numerous colleagues who helped 
shape the results presented in this paper either via comments, discussions, or pointers to related literature. In no 
particular order, we thank Tam\'as Linder, Csaba Szepesv\'ari, Peter Gr\"unwald, Peter Bartlett, Wojciech Kot\l owski, 
Borja Rodr\'iguez G\'alvez, Gautam Goel, Ishaq Aden-Ali, Abhishek Shetty, Nikita Zhivotovskiy, Ilja Kuzborskij, 
Francesco Orabona, Matus Telgarsky, Ohad Shamir, Roi Livni, and Adam Block.

G.~Lugosi acknowledges the support of Ayudas Fundación BBVA a Proyectos de Investigación Científica 2021 
and
the Spanish Ministry of Economy and Competitiveness, Grant t
PID2022-138268NB100.
G.~Neu was supported by the European Research Council (ERC) under the European Union’s Horizon 2020 research 
and innovation programme (Grant agreement No.~950180).
\end{acks}

\bibliographystyle{plainnat}
\bibliography{ngbib,shortconfs}

\begin{thebibliography}{88}
\providecommand{\natexlab}[1]{#1}
\providecommand{\url}[1]{\texttt{#1}}
\expandafter\ifx\csname urlstyle\endcsname\relax
  \providecommand{\doi}[1]{doi: #1}\else
  \providecommand{\doi}{doi: \begingroup \urlstyle{rm}\Url}\fi

\bibitem[Alquier(2021)]{Alq21b}
P.~Alquier.
\newblock Non-exponentially weighted aggregation: regret bounds for unbounded
  loss functions.
\newblock In \emph{International Conference on Machine Learning}, pages
  207--218. PMLR, 2021.

\bibitem[Alquier and Guedj(2018)]{AG18}
P.~Alquier and B.~Guedj.
\newblock Simpler {PAC-Bayesian} bounds for hostile data.
\newblock \emph{Machine Learning}, 107\penalty0 (5):\penalty0 887--902, 2018.

\bibitem[Audibert(2004)]{Aud04}
J.-Y. Audibert.
\newblock \emph{{PAC-Bayesian} statistical learning theory}.
\newblock PhD thesis, Universit{\'e} Paris VI, 2004.

\bibitem[Ball et~al.(1994)Ball, Carlen, and Lieb]{BCL94}
K.~Ball, E.~A. Carlen, and E.~H. Lieb.
\newblock Sharp uniform convexity and smoothness inequalities for trace norms.
\newblock \emph{Inventiones mathematicae}, 115\penalty0 (1):\penalty0 463--482,
  1994.

\bibitem[Bartlett and Mendelson(2002)]{BaMe02}
P.~Bartlett and S.~Mendelson.
\newblock Rademacher and {G}aussian complexities: risk bounds and structural
  results.
\newblock \emph{Journal of Machine Learning Research}, 3:\penalty0 463--482,
  2002.

\bibitem[Bartlett and Mendelson(2006)]{BaMe06}
P.~Bartlett and S.~Mendelson.
\newblock Empirical minimization.
\newblock \emph{Probability Theory and Related Fields}, 135:\penalty0 311--334,
  2006.

\bibitem[Bartlett et~al.(2002)Bartlett, Boucheron, and Lugosi]{BaBoLu01}
P.~Bartlett, S.~Boucheron, and G.~Lugosi.
\newblock Model selection and error estimation.
\newblock \emph{Machine Learning}, 48:\penalty0 85--113, 2002.

\bibitem[Bartlett(1998)]{Bar98}
P.~L. Bartlett.
\newblock The sample complexity of pattern classification with neural networks:
  the size of the weights is more important than the size of the network.
\newblock \emph{IEEE Transactions on Information Theory}, 44\penalty0
  (2):\penalty0 525--536, 1998.

\bibitem[Bassily et~al.(2014)Bassily, Smith, and Thakurta]{BST14}
R.~Bassily, A.~Smith, and A.~Thakurta.
\newblock Private empirical risk minimization: Efficient algorithms and tight
  error bounds.
\newblock In \emph{Foundations of Computer Science (FOCS)}, pages 464--473,
  2014.

\bibitem[B{\'e}gin et~al.(2016)B{\'e}gin, Germain, Laviolette, and Roy]{BGLR16}
L.~B{\'e}gin, P.~Germain, F.~Laviolette, and J.-F. Roy.
\newblock {PAC-Bayesian} bounds based on the {R{\'e}nyi} divergence.
\newblock In \emph{Artificial Intelligence and Statistics}, pages 435--444,
  2016.

\bibitem[Boucheron et~al.(2013)Boucheron, Lugosi, and Massart]{BLM13}
S.~Boucheron, G.~Lugosi, and P.~Massart.
\newblock \emph{Concentration inequalities}.
\newblock Oxford University Press, 2013.

\bibitem[Bousquet and Elisseeff(2002)]{bousquet2002stability}
O.~Bousquet and A.~Elisseeff.
\newblock Stability and generalization.
\newblock \emph{Journal of Machine Learning Research}, 2:\penalty0 499--526,
  2002.

\bibitem[Catoni(2007)]{Cat07}
O.~Catoni.
\newblock {PAC-Bayesian} supervised classification.
\newblock \emph{Lecture Notes-Monograph Series. IMS}, 1277, 2007.

\bibitem[Cesa-Bianchi and Lugosi(2006)]{CBLu06:book}
N.~Cesa-Bianchi and G.~Lugosi.
\newblock \emph{Prediction, Learning, and Games}.
\newblock Cambridge University Press, New York, NY, USA, 2006.

\bibitem[Cesa-Bianchi et~al.(2004)Cesa-Bianchi, Conconi, and Gentile]{CBCoGe04}
N.~Cesa-Bianchi, A.~Conconi, and C.~Gentile.
\newblock On the generalization ability of on-line learning algorithms.
\newblock \emph{IEEE Transactions on Information Theory}, 50:\penalty0
  2050--2057, 2004.

\bibitem[Cesa-Bianchi et~al.(2007)Cesa-Bianchi, Mansour, and Stoltz]{CBMS07}
N.~Cesa-Bianchi, Y.~Mansour, and G.~Stoltz.
\newblock Improved second-order bounds for prediction with expert advice.
\newblock \emph{Machine Learning}, 66\penalty0 (2-3):\penalty0 321--352, 2007.

\bibitem[Chaudhuri et~al.()Chaudhuri, Freund, and Hsu]{CFH09}
K.~Chaudhuri, Y.~Freund, and D.~J. Hsu.
\newblock A parameter-free hedging algorithm.
\newblock pages 297--305.

\bibitem[Chaudhuri et~al.(2011)Chaudhuri, Monteleoni, and Sarwate]{CMS11}
K.~Chaudhuri, C.~Monteleoni, and A.~D. Sarwate.
\newblock Differentially private empirical risk minimization.
\newblock \emph{Journal of Machine Learning Research}, 12:\penalty0 1069--1109,
  2011.

\bibitem[Chernov and Vovk(2010)]{CV10}
A.~Chernov and V.~Vovk.
\newblock Prediction with advice of unknown number of experts.
\newblock In \emph{Proceedings of the Twenty-Sixth Conference on Uncertainty in
  Artificial Intelligence}, pages 117--125, 2010.

\bibitem[Clarkson(1936)]{Cla36}
J.~A. Clarkson.
\newblock Uniformly convex spaces.
\newblock \emph{Transactions of the American Mathematical Society}, 40\penalty0
  (3):\penalty0 396--414, 1936.

\bibitem[Csisz{\'a}r(1963)]{csiszar1963informationstheoretische}
I.~Csisz{\'a}r.
\newblock Eine informationstheoretische ungleichung und ihre anwendung auf den
  beweis der ergodizit{\"a}t von markoffschen ketten.
\newblock \emph{A Magyar Tudom{\'a}nyos Akad{\'e}mia Matematikai Kutat{\'o}
  Int{\'e}zet{\'e}nek K{\"o}zlem{\'e}nyei}, 8\penalty0 (1-2):\penalty0 85--108,
  1963.

\bibitem[Devroye and Wagner(1979)]{devroye1979distribution}
L.~Devroye and T.~J. Wagner.
\newblock Distribution-free inequalities for the deleted and holdout error
  estimates.
\newblock \emph{IEEE Transactions on Information Theory}, 25\penalty0
  (2):\penalty0 202--207, 1979.

\bibitem[Dudley(1984)]{Dud84}
R.~M. Dudley.
\newblock A course on empirical processes.
\newblock In \emph{Ecole d'{\'e}t{\'e} de Probabilit{\'e}s de Saint-Flour
  XII-1982}, pages 1--142. Springer, 1984.

\bibitem[Dwork et~al.(2006{\natexlab{a}})Dwork, Kenthapadi, McSherry, Mironov,
  and Naor]{D+06}
C.~Dwork, K.~Kenthapadi, F.~McSherry, I.~Mironov, and M.~Naor.
\newblock Our data, ourselves: Privacy via distributed noise generation.
\newblock In \emph{Theory and Applications of Cryptographic Techniques
  (EUROCRYPT)}, pages 486--503, 2006{\natexlab{a}}.

\bibitem[Dwork et~al.(2006{\natexlab{b}})Dwork, McSherry, Nissim, and
  Smith]{D+06b}
C.~Dwork, F.~McSherry, K.~Nissim, and A.~D. Smith.
\newblock Calibrating noise to sensitivity in private data analysis.
\newblock In \emph{Theory of Cryptography Conference {(TCC)}}, pages 265--284,
  2006{\natexlab{b}}.

\bibitem[Dziugaite and Roy(2017)]{DR17}
G.~K. Dziugaite and D.~M. Roy.
\newblock Computing nonvacuous generalization bounds for deep (stochastic)
  neural networks with many more parameters than training data.
\newblock 2017.

\bibitem[Foster et~al.(2015)Foster, Rakhlin, and Sridharan]{FRS15}
D.~J. Foster, A.~Rakhlin, and K.~Sridharan.
\newblock Adaptive online learning.
\newblock \emph{Advances in Neural Information Processing Systems}, 28, 2015.

\bibitem[Foster et~al.(2017)Foster, Rakhlin, and Sridharan]{FRS17}
D.~J. Foster, A.~Rakhlin, and K.~Sridharan.
\newblock Zigzag: A new approach to adaptive online learning.
\newblock In \emph{Conference on Learning Theory}, pages 876--924, 2017.

\bibitem[Foster et~al.(2018)Foster, Rakhlin, and Sridharan]{FRS18}
D.~J. Foster, A.~Rakhlin, and K.~Sridharan.
\newblock Online learning: Sufficient statistics and the {Burkholder} method.
\newblock In \emph{Conference On Learning Theory}, pages 3028--3064, 2018.

\bibitem[Freund and Schapire(1997)]{FS97}
Y.~Freund and R.~E. Schapire.
\newblock A decision-theoretic generalization of on-line learning and an
  application to boosting.
\newblock \emph{Journal of Computer and System Sciences}, 55:\penalty0
  119--139, 1997.

\bibitem[Gaillard et~al.(2014)Gaillard, Stoltz, and van
  Erven]{gaillard2014a-second-order}
P.~Gaillard, G.~Stoltz, and T.~van Erven.
\newblock A second-order bound with excess losses.
\newblock In \emph{Conference on Learning Theory}, pages 176--196, 2014.

\bibitem[Gnedenko and Kolmogorov(1954)]{gnedenko1954limit}
B.~V. Gnedenko and A.~N. Kolmogorov.
\newblock Limit distributions for sums of independent random variables.
\newblock \emph{Am. J. Math}, 105, 1954.

\bibitem[Gr\"unwald et~al.(2021)Gr\"unwald, Steinke, and Zakynthinou]{GSZ21}
P.~Gr\"unwald, T.~Steinke, and L.~Zakynthinou.
\newblock {PAC-Bayes, MAC-Bayes} and conditional mutual information: Fast rate
  bounds that handle general {VC} classes.
\newblock In \emph{Conference on Learning Theory}, pages 2217--2247, 2021.

\bibitem[Haghifam et~al.(2021)Haghifam, Dziugaite, Moran, and Roy]{HDMR21}
M.~Haghifam, G.~K. Dziugaite, S.~Moran, and D.~Roy.
\newblock Towards a unified information-theoretic framework for generalization.
\newblock \emph{Advances in Neural Information Processing Systems},
  34:\penalty0 26370--26381, 2021.

\bibitem[Hardt et~al.(2016)Hardt, Recht, and Singer]{HRS16}
M.~Hardt, B.~Recht, and Y.~Singer.
\newblock Train faster, generalize better: Stability of stochastic gradient
  descent.
\newblock In \emph{International Conference on Machine Learning (ICML)}, pages
  1225--1234, 2016.

\bibitem[Hellstr{\"o}m and Durisi(2020)]{HD20a}
F.~Hellstr{\"o}m and G.~Durisi.
\newblock Generalization bounds via information density and conditional
  information density.
\newblock \emph{IEEE Journal on Selected Areas in Information Theory},
  1\penalty0 (3):\penalty0 824--839, 2020.

\bibitem[Hiriart-Urruty and Lemar{\'e}chal(2013)]{HL13}
J.-B. Hiriart-Urruty and C.~Lemar{\'e}chal.
\newblock \emph{Convex analysis and minimization algorithms I: Fundamentals},
  volume 305.
\newblock Springer science \& business media, 2013.

\bibitem[Jiang et~al.(2020)Jiang, Neyshabur, Mobahi, Krishnan, and
  Bengio]{JNMKB20}
Y.~Jiang, B.~Neyshabur, H.~Mobahi, D.~Krishnan, and S.~Bengio.
\newblock Fantastic generalization measures and where to find them.
\newblock In \emph{International Conference on Learning Representations}, 2020.

\bibitem[Kakade et~al.(2008)Kakade, Sridharan, and Tewari]{KST08}
S.~M. Kakade, K.~Sridharan, and A.~Tewari.
\newblock On the complexity of linear prediction: Risk bounds, margin bounds,
  and regularization.
\newblock \emph{Advances in neural information processing systems}, 21, 2008.

\bibitem[Kemperman(1974)]{Kem74}
J.~Kemperman.
\newblock On the {S}hannon capacity of an arbitrary channel.
\newblock In \emph{Indagationes Mathematicae (Proceedings)}, volume~77, pages
  101--115. North-Holland, 1974.

\bibitem[Koltchinskii(2001)]{Kol00a}
V.~Koltchinskii.
\newblock Rademacher penalties and structural risk minimization.
\newblock \emph{IEEE Transactions on Information Theory}, 47:\penalty0
  1902--1914, 2001.

\bibitem[Koolen and Van~Erven(2015)]{KvE15}
W.~M. Koolen and T.~Van~Erven.
\newblock Second-order quantile methods for experts and combinatorial games.
\newblock In \emph{Conference on Learning Theory}, pages 1155--1175, 2015.

\bibitem[Langford and Caruana(2001)]{LC01}
J.~Langford and R.~Caruana.
\newblock (not) bounding the true error.
\newblock \emph{Advances in Neural Information Processing Systems}, 14, 2001.

\bibitem[Littlestone and Warmuth(1994)]{LW94}
N.~Littlestone and M.~Warmuth.
\newblock The weighted majority algorithm.
\newblock \emph{Information and Computation}, 108:\penalty0 212--261, 1994.

\bibitem[Livni and Moran(2020)]{LM20}
R.~Livni and S.~Moran.
\newblock A limitation of the {PAC-Bayes} framework.
\newblock \emph{Advances in Neural Information Processing Systems},
  33:\penalty0 20543--20553, 2020.

\bibitem[Lugosi and Neu(2022)]{LN22}
G.~Lugosi and G.~Neu.
\newblock Generalization bounds via convex analysis.
\newblock In \emph{Proceedings of the 35th Conference on Learning Theory
  (COLT)}, pages 3524--3546, 2022.

\bibitem[Luo and Schapire(2015)]{LS15}
H.~Luo and R.~E. Schapire.
\newblock Achieving all with no parameters: Adanormalhedge.
\newblock pages 1286--1304, 2015.

\bibitem[McAllester(2013)]{McA13}
D.~McAllester.
\newblock A {PAC-Bayesian} tutorial with a dropout bound.
\newblock \emph{arXiv preprint arXiv:1307.2118}, 2013.

\bibitem[McAllester(1998)]{McA98}
D.~A. McAllester.
\newblock Some {PAC-Bayesian} theorems.
\newblock In \emph{Proceedings of the eleventh annual conference on
  Computational Learning Theory (COLT)}, pages 230--234, 1998.

\bibitem[Mhammedi et~al.(2019)Mhammedi, Gr{\"u}nwald, and Guedj]{MGG19}
Z.~Mhammedi, P.~Gr{\"u}nwald, and B.~Guedj.
\newblock {PAC-Bayes} un-expected {Bernstein} inequality.
\newblock \emph{Advances in Neural Information Processing Systems}, 32, 2019.

\bibitem[Mukherjee et~al.(2006)Mukherjee, Niyogi, Poggio, and
  Rifkin]{mukherjee2006learning}
S.~Mukherjee, P.~Niyogi, T.~Poggio, and R.~Rifkin.
\newblock Learning theory: stability is sufficient for generalization and
  necessary and sufficient for consistency of empirical risk minimization.
\newblock \emph{Advances in Computational Mathematics}, 25:\penalty0 161--193,
  2006.

\bibitem[Nagarajan and Kolter(2019)]{NK19}
V.~Nagarajan and J.~Z. Kolter.
\newblock Uniform convergence may be unable to explain generalization in deep
  learning.
\newblock \emph{Advances in Neural Information Processing Systems}, 32, 2019.

\bibitem[Neu et~al.(2021)Neu, Dziugaite, Haghifam, and Roy]{NDHR21}
G.~Neu, G.~K. Dziugaite, M.~Haghifam, and D.~M. Roy.
\newblock Information-theoretic generalization bounds for stochastic gradient
  descent.
\newblock In \emph{Proceedings of the 34th Conference on Learning Theory
  (COLT)}, pages 3526--3545, 2021.

\bibitem[Neyshabur et~al.(2017)Neyshabur, Bhojanapalli, McAllester, and
  Srebro]{NBMS17}
B.~Neyshabur, S.~Bhojanapalli, D.~McAllester, and N.~Srebro.
\newblock Exploring generalization in deep learning.
\newblock \emph{Advances in neural information processing systems}, 30, 2017.

\bibitem[Neyshabur et~al.(2018)Neyshabur, Bhojanapalli, and Srebro]{NBS18}
B.~Neyshabur, S.~Bhojanapalli, and N.~Srebro.
\newblock A pac-bayesian approach to spectrally-normalized margin bounds for
  neural networks.
\newblock In \emph{International Conference on Learning Representations}, 2018.

\bibitem[Nielsen and Nock(2011)]{NN11}
F.~Nielsen and R.~Nock.
\newblock On {R\'enyi and Tsallis} entropies and divergences for exponential
  families.
\newblock \emph{arXiv preprint arXiv:1105.3259}, 2011.

\bibitem[Orabona(2019)]{Ora19}
F.~Orabona.
\newblock A modern introduction to online learning.
\newblock \emph{arXiv preprint arXiv:1912.13213}, 2019.

\bibitem[Orabona and Jun(2021)]{OJ21}
F.~Orabona and K.-S. Jun.
\newblock Tight concentrations and confidence sequences from the regret of
  universal portfolio.
\newblock \emph{arXiv preprint arXiv:2110.14099}, 2021.

\bibitem[Orabona and P{\'a}l(2016)]{OP16}
F.~Orabona and D.~P{\'a}l.
\newblock Coin betting and parameter-free online learning.
\newblock In \emph{Advances in Neural Information Processing Systems}, pages
  577--585, 2016.

\bibitem[Pearson(1900)]{Pea00}
K.~Pearson.
\newblock On the criterion that a given system of deviations from the probable
  in the case of a correlated system of variables is such that it can be
  reasonably supposed to have arisen from random sampling.
\newblock \emph{The London, Edinburgh, and Dublin Philosophical Magazine and
  Journal of Science}, 50\penalty0 (302):\penalty0 157--175, 1900.

\bibitem[Rakhlin(2009)]{Rak09}
A.~Rakhlin.
\newblock Lecture notes on online learning.
\newblock 2009.

\bibitem[Rakhlin and Sridharan(2013{\natexlab{a}})]{RS13}
A.~Rakhlin and K.~Sridharan.
\newblock Online learning with predictable sequences.
\newblock In \emph{Conference on Learning Theory}, pages 993--1019,
  2013{\natexlab{a}}.

\bibitem[Rakhlin and Sridharan(2013{\natexlab{b}})]{RS13b}
A.~Rakhlin and K.~Sridharan.
\newblock Optimization, learning, and games with predictable sequences.
\newblock In \emph{Advances in Neural Information Processing Systems}, pages
  3066--3074, 2013{\natexlab{b}}.

\bibitem[Rakhlin and Sridharan(2017)]{RS17}
A.~Rakhlin and K.~Sridharan.
\newblock On equivalence of martingale tail bounds and deterministic regret
  inequalities.
\newblock In \emph{Conference on Learning Theory}, pages 1704--1722. PMLR,
  2017.

\bibitem[R{\'e}nyi(1961)]{renyi1961measures}
A.~R{\'e}nyi.
\newblock On measures of entropy and information.
\newblock In \emph{Proceedings of the fourth Berkeley symposium on mathematical
  statistics and probability, volume 1: contributions to the theory of
  statistics}, volume~4, pages 547--562. University of California Press, 1961.

\bibitem[Rodr{\'\i}guez-G{\'a}lvez et~al.(2021)Rodr{\'\i}guez-G{\'a}lvez,
  Bassi, Thobaben, and Skoglund]{RBTS21}
B.~Rodr{\'\i}guez-G{\'a}lvez, G.~Bassi, R.~Thobaben, and M.~Skoglund.
\newblock Tighter expected generalization error bounds via {Wasserstein}
  distance.
\newblock In \emph{Advances in Neural Information Processing Systems}, 2021.

\bibitem[Russo and Zou(2016)]{RZ16}
D.~Russo and J.~Zou.
\newblock Controlling bias in adaptive data analysis using information theory.
\newblock In \emph{Artificial Intelligence and Statistics}, pages 1232--1240,
  2016.

\bibitem[Russo and Zou(2019)]{RZ19}
D.~Russo and J.~Zou.
\newblock How much does your data exploration overfit? controlling bias via
  information usage.
\newblock \emph{IEEE Transactions on Information Theory}, 66\penalty0
  (1):\penalty0 302--323, 2019.

\bibitem[Sason and Verd{\'u}(2016)]{SV16}
I.~Sason and S.~Verd{\'u}.
\newblock $f$-divergence inequalities.
\newblock \emph{IEEE Transactions on Information Theory}, 62\penalty0
  (11):\penalty0 5973--6006, 2016.

\bibitem[Seldin et~al.(2012)Seldin, Laviolette, Cesa-Bianchi, Shawe-Taylor, and
  Auer]{SLCSA12}
Y.~Seldin, F.~Laviolette, N.~Cesa-Bianchi, J.~Shawe-Taylor, and P.~Auer.
\newblock {PAC-Bayesian} inequalities for martingales.
\newblock \emph{IEEE Transactions on Information Theory}, 58\penalty0
  (12):\penalty0 7086--7093, 2012.

\bibitem[Shafer and Vovk(2001)]{SV01}
G.~Shafer and V.~Vovk.
\newblock \emph{Probability and finance: it's only a game!}, volume 491.
\newblock John Wiley \& Sons, 2001.

\bibitem[Shalev-Shwartz(2012)]{SS12}
S.~Shalev-Shwartz.
\newblock Online learning and online convex optimization.
\newblock \emph{Foundations and Trends in Machine Learning}, 4\penalty0
  (2):\penalty0 107--194, 2012.

\bibitem[Shalev-Shwartz et~al.(2010)Shalev-Shwartz, Shamir, Srebro, and
  Sridharan]{shalev2010learnability}
S.~Shalev-Shwartz, O.~Shamir, N.~Srebro, and K.~Sridharan.
\newblock Learnability, stability and uniform convergence.
\newblock \emph{Journal of Machine Learning Research}, 11:\penalty0 2635--2670,
  2010.

\bibitem[Shawe-Taylor and Williamson(1997)]{STW97}
J.~Shawe-Taylor and R.~C. Williamson.
\newblock A {PAC} analysis of a {Bayesian} estimator.
\newblock In \emph{Proceedings of the tenth annual conference on Computational
  learning theory}, pages 2--9, 1997.

\bibitem[Steinke and Zakynthinou(2020)]{SZ20}
T.~Steinke and L.~Zakynthinou.
\newblock Reasoning about generalization via conditional mutual information.
\newblock In \emph{Proceedings of the 33rd Conference on Learning Theory
  (COLT)}, 2020.

\bibitem[Tolstikhin and Seldin(2013)]{TS13}
I.~O. Tolstikhin and Y.~Seldin.
\newblock {PAC-Bayes-empirical-Bernstein} inequality.
\newblock \emph{Advances in Neural Information Processing Systems}, 26, 2013.

\bibitem[van~der Hoeven et~al.(2018)van~der Hoeven, van Erven, and
  Kot{\l}owski]{HEK18}
D.~van~der Hoeven, T.~van Erven, and W.~Kot{\l}owski.
\newblock The many faces of exponential weights in online learning.
\newblock In \emph{Conference On Learning Theory}, pages 2067--2092, 2018.

\bibitem[van Erven et~al.(2015)van Erven, Gr{\"u}nwald, Mehta, Reid, and
  Williamson]{vEGMRW15}
T.~van Erven, P.~D. Gr{\"u}nwald, N.~A. Mehta, M.~D. Reid, and R.~C.
  Williamson.
\newblock Fast rates in statistical and online learning.
\newblock \emph{Journal of Machine Learning Research}, 16:\penalty0 1793--1861,
  2015.

\bibitem[Vapnik and Chervonenkis(1974)]{VaCh74a}
V.~Vapnik and A.~Chervonenkis.
\newblock \emph{Theory of Pattern Recognition}.
\newblock Nauka, Moscow, 1974.
\newblock (in Russian); German translation: {\em Theorie der Zeichenerkennung},
  Akademie Verlag, Berlin, 1979.

\bibitem[Villani(2003)]{Vil03}
C.~Villani.
\newblock \emph{Topics in optimal transportation}, volume~58.
\newblock American Mathematical Soc., 2003.

\bibitem[Vovk(1990)]{Vov90}
V.~Vovk.
\newblock Aggregating strategies.
\newblock In \emph{Proceedings of the third annual workshop on Computational
  learning theory (COLT)}, pages 371--386, 1990.

\bibitem[Wang et~al.(2019)Wang, Diaz, Santos~Filho, and Calmon]{WDSC19}
H.~Wang, M.~Diaz, J.~C.~S. Santos~Filho, and F.~P. Calmon.
\newblock An information-theoretic view of generalization via {Wasserstein
  }distance.
\newblock In \emph{2019 IEEE International Symposium on Information Theory
  (ISIT)}, pages 577--581. IEEE, 2019.

\bibitem[Waudby-Smith and Ramdas(2020)]{WSR20}
I.~Waudby-Smith and A.~Ramdas.
\newblock Estimating means of bounded random variables by betting.
\newblock \emph{arXiv preprint arXiv:2010.09686}, 2020.

\bibitem[Williamson et~al.(2000)Williamson, Smola, and Sch{\"o}lkopf]{WSS00}
R.~C. Williamson, A.~J. Smola, and B.~Sch{\"o}lkopf.
\newblock Entropy numbers of linear function classes.
\newblock In \emph{Proceedings of the 13th Conference on Learning Theory
  (COLT)}, pages 309--319, 2000.

\bibitem[Xu and Raginsky(2017)]{XR17}
A.~Xu and M.~Raginsky.
\newblock Information-theoretic analysis of generalization capability of
  learning algorithms.
\newblock In \emph{Advances in Neural Information Processing Systems}, pages
  2524--2533, 2017.

\bibitem[Yao(1977)]{Yao77}
A.~C. Yao.
\newblock Probabilistic computations: Toward a unified measure of complexity.
\newblock In \emph{18th Annual Symposium on Foundations of Computer Science},
  pages 222--227, 1977.

\bibitem[Zhang et~al.(2018)Zhang, Liu, and Tao]{ZLT18}
J.~Zhang, T.~Liu, and D.~Tao.
\newblock An optimal transport view on generalization.
\newblock \emph{arXiv preprint arXiv:1811.03270}, 2018.

\bibitem[Zhang(2002)]{Zha02}
T.~Zhang.
\newblock Covering number bounds of certain regularized linear function
  classes.
\newblock \emph{Journal of Machine Learning Research}, 2\penalty0
  (Mar):\penalty0 527--550, 2002.

\end{thebibliography}

%
%

\newpage

\appendix

\section{Regret bounds}\label{app:regret}
\subsection{Exponentially weighted averaging}\label{app:MWA}
The following regret bound is an adaptation of the classic results of \citet{Vov90,LW94,FS97}.
\begin{theorem}\label{thm:MWA}
Consider the exponentially weighted average forecaster defined via the iteration
\[
 P_{t+1} = \argmin_{P\in\Delta_{\Ww}} \ev{\iprod{P}{c_t} - \frac{1}{\eta} \DDKL{P}{P_t}}~.
\]
For any sequence of cost functions $c_1,c_2,\dots,c_n$, the regret of this method satisfies
\[
 \regret_n(P^*) \le \frac{\DDKL{P^*}{P_1}}{\eta} + \frac{\eta}{2} \sum_{t=1}^n \infnorm{c_t}^2.
\]
\end{theorem}
\begin{proof}
The proof is based on studying a potential function $\Phi$ defined for all $c\in\real^{\Ww}$ as
 \[
  \Phi(c) = \frac 1\eta \log \intW e^{-\eta c(w)} \dd P_1(w).
 \]
In particular, we consider $\Phi(\sum_{t=1}^n c_t)$ and notice that it is related to the total cost of the 
comparator $P^*$ as follows:
\[
 \Phi\pa{\sum_{t=1}^n c_t} = \frac 1\eta \log \intW e^{- \eta \sum_{t=1}^n c_t(W)} \dd P_1(w) 
 \ge -\sum_{t=1}^n \iprod{P^*}{c_t} - \frac{1}{\eta} \DDKL{P^*}{P_1},
\]
where the inequality is the Donsker--Varadhan variational formula (cf.~Section 4.9 in \citet{BLM13}). On the other 
hand, we have
\begin{align*}
 \Phi\pa{\sum_{t=1}^n c_t} &= \sum_{t=1}^{n} \pa{\Phi\pa{\sum_{k=1}^{t} c_k} - \Phi\pa{\sum_{k=1}^{t-1} c_k}}
\\ &= \sum_{t=1}^{n} \frac{1}{\eta} \log \frac{\intW e^{-\eta \sum_{k=1}^{t} c_k(w)}\dd P_1(w)}{\intW e^{-\eta 
\sum_{k=1}^{t-1} c_t(w)}\dd P_1(w)}
\\
 &= \sum_{t=1}^{n} \frac{1}{\eta} \log \intW \frac{e^{-\eta \sum_{k=1}^{t-1} c_k(w)}}{{\intW e^{-\eta 
\sum_{k=1}^{t-1} c_k(w)}\dd P_1(w)}} \cdot e^{-\eta c_t(w)}\dd P_1(w)
\\
 &= \sum_{t=1}^{n} \frac{1}{\eta} \log \intW e^{-\eta c_t(w)}\dd P_t(w) \label{eq:MWA_potential_upper}
\end{align*}
Finally, we notice that the term appearing on the right-hand side can be bounded using Hoeffding's lemma (see, 
e.g., Lemma~2.2 in \citet{BLM13}):
\begin{equation}\label{eq:MWA_potential_upper}
 \frac {1}{\eta} \log \intW e^{-\eta c_t(w)}\dd P_t(w) = - \iprod{P_t}{c_t} + \frac 1\eta \intW e^{-\eta \pa{c_t(w) 
- \iprod{P_t}{c_t}}}\dd P_t(w) \le - \iprod{P_t}{c_t} + \frac{\eta \infnorm{c_t}^2}{2}.
\end{equation}
Plugging this inequality back into the previous calculations concludes the proof.
\end{proof}

\subsection{Optimistic Second-Order EWA}\label{app:second-order}
Let us now consider an ``optimistic'' version of a EWA-based method that uses a guess $g_t$ of $c_t$ when 
playing its action $P_t$. This algorithm calculates two sequences of updates: first, an auxiliary 
distribution $\tP_t$ as
\[
 \frac{\dd \tP_{t+1}}{\dd \tP_t}(w) = \frac{e^{-\eta c_t(w) - \eta^2 (c_t(w) - g_t(w))^2}}{\int_{\Ww} e^{-\eta c_t(w') 
- \eta^2 (c_t(w) - g_t(w))^2} \dd \tP_t(w')},
\]
and the actual update calculated as 
\[
 \frac{\dd P_{t+1}}{\dd \tP_{t+1}}(w) = \frac{e^{-\eta g_{t+1}(w)}}{{\int_\Ww e^{-\eta g_{t+1}(w')} \dd \tP_{t+1} 
(w')}}.
\]
The second-order adjustment appearing in the auxiliary update sequence is commonly used in the online learning 
literature to achieve data-dependent bounds \citep{CBMS07,gaillard2014a-second-order,KvE15}, and is the component that 
enables us to prove a strong comparator-dependent regret bound. The idea of using an auxiliary update sequence is 
inspired by the ``optimistic online learning'' algorithms of \citet{RS13,RS13b}, and allows us to achieve fully 
data-dependent bounds by eliminating the test error from the generalization bounds. 
The algorithm defined above satisfies the following regret bound.
\begin{theorem}\label{thm:second-order}
For any $\eta\in\left[0,\frac 12 \right]$ and any sequence of cost functions $c_1,c_2,\dots,c_n$ and predictions 
$g_1,g_2,\dots,g_n$,  the regret of the 
optimistic second-order EWA forecaster defined above satisfies
\[
 \sum_{t=1}^n \iprod{P_t - P^*}{c_t} \le \frac{\DDKL{P^*}{P_1}}{\eta} + \eta \sum_{t=1}^n \iprod{P^*}{(c_t - g_t)^2}.
\] 
\end{theorem}
We are going to use this theorem with $c_t(w) = \ell(w,Z_t) - \EE{\ell(w,Z')}$ and $g_t(w) = - \EE{\ell(w,Z')} \le 0$ 
so the last term on the right-hand side is negative, which will help us deduce a fast rate from the regret bound.
\begin{proof}
The proof follows from similar arguments as used in the proof of Theorem~\ref{thm:MWA}. In particular, we 
introduce the auxiliary notation $\tc_t(w) = c_t(w) - \eta (c_t(w) - g_t(w))^2$ and study the potential 
$\Phi(\sum_{t=1}^n \tc_t)$. Lower-bounding the potential gives
\[
 \Phi\pa{\sum_{t=1}^n \tc_t} = \frac 1\eta \log \intW e^{- \eta \sum_{t=1}^n \tc_t(W)} \dd \tP_1(w) 
 \ge -\sum_{t=1}^n \iprod{P^*}{\tc_t} - \frac{1}{\eta} \DDKL{P^*}{\tP_1}.
\]
On the other hand, we have
\[
\Phi\pa{\sum_{t=1}^n \tc_t} = \sum_{t=1}^{n} \frac{1}{\eta} \log \intW e^{-\eta \tc_t(w)}\dd \tP_t(w).
\]
Noticing that $\frac{\dd \tP_t}{\dd P_t} (w) = \frac{e^{\eta g_t(w)}}{\int_{\Ww} e^{\eta g_t(w')} \dd P_t(w')}$, we can 
upper bound each term in the above sum as
\begin{align*}
&\frac {1}{\eta} \log \intW e^{-\eta \wt{c}_t(w)}\dd \tP_t(w) = \frac {1}{\eta} \log \intW e^{-\eta c_t(w) - \eta^2 
\pa{c_t(w) - g_t(w)}^2} \frac{\dd \tP_t}{\dd P_t} (w) \dd P_t(w)
\\
&\qquad\qquad= \frac {1}{\eta} \log \intW e^{-\eta (c_t(w)-g_t(w)) - \eta^2 
\pa{c_t(w) - g_t(w)}^2}\dd P_t(w) - \frac 1\eta \log \int_{\Ww} e^{\eta g_t(w)}\dd 
P_t(w)
\\
&\qquad\qquad\le \frac {1}{\eta} \log \intW \pa{1 - \eta \pa{c_t(w) - g_t(w)}}   \dd P_t(w)
- \iprod{P_t}{g_t}
\\
&\qquad\qquad\le - \iprod{P_t}{c_t},
\end{align*}
where we have used the inequality $e^{-x - x^2} \le 1-x$ that holds for all $x\le \frac 12$ (which is 
ensured by our choice of $\eta \le \frac 12$ and the boundedness of the cost function), and Jensen's inequality that 
implies $\iprod{P_t}{g_t} \le \frac 1\eta \log \int_{\Ww} e^{\eta g_t(w)}\dd 
P_t(w)$.
Putting the two bounds together proves the statement.
\end{proof}

\subsection{Follow the Regularized Leader}\label{app:FTRL}
We recall that the predictions of the FTRL algorithm are defined as
\[
 P_{t} = \argmin_{P\in \Delta_{\Ww}} \ev{\iprod{P}{\sum_{k=1}^{t-1} c_k} + \frac{1}{\eta} h(P)}~.
\]
For simplicity, we use the notation $C_t = \sum_{k=1}^t c_k$ hereafter.
We first show that, under the conditions we have assumed in the main text (properness and strong convexity), the 
minimum exists and is unique. For simplicity, we use the notation $\Psi_t = \iprod{\cdot}{C_{t-1}} + 
\frac{1}{\eta} h$. We denote the effective domain 
of $h$ by $\Gamma_h = \ev{P\in\Delta_{\Ww}: h(P) < +\infty}$ and note that the condition that $h$ is proper and lower 
semicontinuous implies that $\Psi_t$ is also proper and lower semicontinuous. Together with the compactness of 
$\Delta_{\Ww}$, this implies the existence of a minimum. In order to show unicity, let us suppose that $P_t$ and $P_t'$ 
are both minimizers of $\Psi_t$. Then, by convexity, we have for all $\lambda\in[0,1]$ that $\lambda P_t + 
(1-\lambda)P_t' \in \argmin_{P} \Psi_t(P)$. However, by strong convexity of $h$, we know that $\Psi_t$ is also strongly 
convex and thus
\begin{align*}
 \min_{P\in\Delta_{\Ww}}\Psi_t(P) &= \Psi_t(\lambda P_t + (1-\lambda)P_t') \le \lambda \Psi_t(P_t) + 
(1-\lambda)\Psi_t(P_t') - \frac{\alpha \lambda (1-\lambda)}{2} \norm{P_t - P_t'}^2 
\\
&= \min_{P\in\Delta_{\Ww}}\Psi_t(P) - \frac{\alpha \lambda (1-\lambda)}{2} \norm{P_t - P_t'}^2.
\end{align*}
Hence, equality is only possible when $\norm{P_t - P_t'} = 0$, or, equivalently, $P_t = P_t'$. This shows that the 
minimum is indeed unique.

The analysis below uses a few other concepts from convex analysis. A key notion is the Legendre--Fenchel conjugate 
of the convex function $h$ denoted as $h^*:\real^{\Ww}\ra\real$, mapping a function $f$ to $h^*(f) = 
\max_{P\in\Delta_{\Ww}} \ev{\iprod{P}{f} - h(P)}$. The subdifferential of a convex functional 
$U: \real^{\Ww}\ra\real$ at $f\in\real^{\Ww}$ is defined as the set of signed measures
\[
 \partial U(f) = \ev{P\in \mathcal{Q}: U(g) \ge U(f) + \iprod{P}{g-f} \,\,\, (\forall 
g\in\real^{\Ww})},
\]
and the associated (generalized) Bregman divergence is defined as
\[
 \DDU{g}{f} = U(g) - U(f) + \sup_{P \in \partial U(f)} \iprod{P}{f - g},
\]
where the supremum is introduced to resolve the ambiguity of the subdifferential. Notice that this is a convex 
function of $g$, being a sum of a convex function and a supremum of affine functions, and that $\DDU{g}{f}\ge 0$ for 
all $f$ and $g$ due to convexity of $U$. We finally note that the $\alpha$-strong convexity of a function $h$ can be 
seen to be equivalent to the condition that $\DDh{P}{P'} \ge \frac{\alpha}{2}\norm{P - P'}^2$ hold for all $P,P'$, and 
$\beta$-smoothness is equivalent to requiring $\DDh{P}{P'} \le \frac{\beta}{2} \norm{P-P'}^2$.

Having established these basic facts, we are now ready to state a regret bound for the above algorithm. 
\begin{theorem}\label{thm:FTRL_regret}
Suppose that $h$ is $\alpha$-strongly convex with respect to the norm $\norm{\cdot}$. Then, 
for any $\eta>0$ and any sequence of cost functions $c_1,c_2,\dots,c_n$,  the regret of the FTRL 
algorithm defined above satisfies
\[
 \regret_n(P^*) \le \frac{h\pa{P^*} - h\pa{P_1}}{\eta} + \frac{\eta}{2\alpha}\sum_{t=1}^n \norm{c_t}_*^2.
\]
\end{theorem}
We are not aware of any other paper that would state this theorem at this level of generality, as all existing 
proofs we know of concern the case of finite-dimensional vector spaces. To our knowledge, the only exception is the 
result of \citet{Alq21b}, who considers the special case where $h$ is an $f$-divergence. Our own proof below follows 
straightforwardly from standard arguments, such as the ones in \citet{Rak09,SS12,Ora19}. 
\begin{proof}
The proof is based on studying a potential function $\Phi$ defined for all $c\in\real^{\Ww}$ as
 \[
  \Phi(c) = \max_{P\in\Delta_{\Ww}} \ev{-\iprod{P}{c} - \frac{1}{\eta} h(P)} = \frac{1}{\eta} h^*(- \eta c)~.
 \]
In particular, we consider $\Phi(C_n)$ and notice that it is related to the total cost of the 
comparator $P^*$ as follows:
\[
 \Phi\pa{C_n} = \max_{P\in\Delta_{\Ww}} \ev{-\iprod{P}{C_n} - \frac{1}{\eta} h(P)}
 \ge -\sum_{t=1}^n \iprod{P^*}{c_t} - \frac{h(P^*)}{\eta}~.
\]
On the other hand, we have
\begin{align*}
\Phi\pa{C_n} &= \sum_{t=1}^{n} \pa{\Phi\pa{C_t} - \Phi\pa{C_{t-1}}} + \Phi(0)
\\ 
&= \sum_{t=1}^{n} \pa{\DDPhi{C_t}{C_{t-1}} - \sup_{P\in\partial \Phi(C_{t-1})} \iprod{P}{C_{t-1} - C_t}} - 
\frac{h(P_1)}{\eta}
\\
&\le \sum_{t=1}^{n} \Bpa{\DDPhi{C_t}{C_{t-1}} - \iprod{P_t}{c_t}} - \frac{h(P_1)}{\eta},
\end{align*}
where in the last step we have used the fact that $-P_t \in \partial \Phi(C_{t-1})$. Indeed, this follows from the 
definition of the algorithm:
\begin{align*}
 \Phi(C_{t-1}) - \iprod{P_t}{c_t} &= \max_{P\in\Delta_{\Ww}} \ev{-\iprod{P}{C_{t-1}} - \frac{1}{\eta} h(P)} 
 - \iprod{P_t}{c_t} = -\iprod{P_t}{C_{t}} - \frac{1}{\eta} h(P_t) 
 \\
 &\le \max_{P\in\Delta_{\Ww}} \ev{-\iprod{P}{C_{t}} - \frac{1}{\eta} h(P)} = \Phi(C_{t})~.
\end{align*}
Putting the above inequalities together, we obtain the following bound on the regret:
\[
 \sum_{t=1}^{n} \iprod{P^* - P_t}{c_t} \le \frac{h(P^*) - h(P_1)}{\eta} + \sum_{t=1}^{n} \DDPhi{C_t}{C_{t-1}}.
\]
Finally, we note that $\Phi$ is the Legendre--Fenchel conjugate of $P\mapsto \frac{1}{\eta} h(\eta P)$, which is an 
$\eta/\alpha$ strongly convex function of its argument. Thus, we use a classic duality property between the 
regularizer $h$ and $h^*$ (proved in Appendix~\ref{app:seminorm-smoothness} for completeness) to show that 
\[
 \DDPhi{C_t}{C_{t-1}} \le \frac{\eta \norm{c_t}_*^2}{2\alpha}.
\]
This completes the proof.
\end{proof}

\subsubsection{Optimistic FTRL}
We also consider an ``optimistic'' version of FTRL that makes use of a sequence of hints $g_t\in\real^\Ww$, by choosing 
its updates according to the assignment
\[
 P_{t} = \argmin_{P\in \Delta_{\Ww}} \ev{\iprod{P}{g_t + \sum_{k=1}^{t-1} c_k} + \frac{1}{\eta} h(P)}~.
\]
A similar method has been proposed and analyzed by \citet{RS13,RS13b}. The following performance guarantee is easily 
obtained by a series of simple adjustments to the proof of Theorem~\ref{thm:FTRL_regret} presented above.
\begin{theorem}\label{thm:optimistic_FTRL_regret}
Suppose that $h$ is $\alpha$-strongly convex with respect to the norm $\norm{\cdot}$. Then, 
for any $\eta>0$ and any sequence of cost functions $c_1,c_2,\dots,c_n$ and predictions 
$g_1,g_2,\dots,g_n$,  the regret of the optimistic FTRL algorithm defined above satisfies
\[
 \regret_n(P^*) \le \frac{h\pa{P^*} - h\pa{P_1}}{\eta} + \frac{\eta}{2\alpha}\sum_{t=1}^n \norm{c_t - g_t}_*^2~.
\]
\end{theorem}
\begin{proof}
The proof is based on studying a potential function $\Phi$ defined for all $c\in\real^{\Ww}$ as
 \[
  \Phi(c) = \max_{P\in\Delta_{\Ww}} \ev{-\iprod{P}{c} - \frac{1}{\eta} h(P)} = \frac{1}{\eta} h^*(- \eta c)~.
 \]
In particular, we consider $\Phi(C_n)$ and notice that it is related to the total cost of the 
comparator $P^*$ as follows:
\[
 \Phi\pa{C_n} = \max_{P\in\Delta_{\Ww}} \ev{-\iprod{P}{C_n} - \frac{1}{\eta} h(P)}
 \ge -\sum_{t=1}^n \iprod{P^*}{c_t} - \frac{h(P^*)}{\eta}~.
\]
On the other hand, we have
\begin{equation}\label{eq:optimistic_potential}
\Phi\pa{C_n} = \sum_{t=1}^{n} \pa{\Phi\pa{C_t} - \Phi\pa{C_{t-1} + g_t} + \Phi\pa{C_{t-1} + g_t} - \Phi\pa{C_{t-1}}} + 
\Phi(0).
\end{equation}
We bound the two key terms arising in the above expression as follows, recalling that $-P_t \in \partial 
\Phi(C_{t-1}+g_t)$. First, we have
\begin{align*}
 \Phi\pa{C_t} - \Phi\pa{C_{t-1} + g_t} &\le \DDPhi{C_t}{C_{t-1}+g_t} - \sup_{P\in\partial \Phi(C_{t-1} + g_t)} 
\iprod{P}{C_{t-1} - C_t + g_t}
\\
&\le \DDPhi{C_t}{C_{t-1}+g_t} - \iprod{P_t}{c_t - g_t}~.
\end{align*}
The remaining term is treated by exploiting the convexity of $\Phi$ that guarantees
\[
 \Phi\pa{C_{t-1}} \ge \Phi\pa{C_{t-1} + g_t} - \sup_{P\in\partial \Phi(C_{t-1}+g_t)} \iprod{P}{g_t} \ge 
 \Phi\pa{C_{t-1} + g_t} + \iprod{P_t}{g_t}
\]
by the definition of the subdifferential $\partial \Phi(C_{t-1}+g_t)$. Putting these inequalities together with 
Equation~\eqref{eq:optimistic_potential}, we obtain
\[
 \Phi\pa{C_n} \le \sum_{t=1}^{n} \Bpa{\DDPhi{C_t}{C_{t-1}+g_t} - \iprod{P_t}{c_t}} - \frac{h(P_1)}{\eta}~,
\]
The proof is concluded by putting everything together and using the smoothness property of $\Phi$ implied by the strong 
convexity of $h$ (cf.~Appendix~\ref{app:seminorm-smoothness}):
\[
 \DDPhi{C_t}{C_{t-1} + g_t} \le \frac{\eta \norm{c_t-g_t}_*^2}{2\alpha}~.
\]
\end{proof}

\subsubsection{The proof of Corollary~\ref{cor:pnorm}}\label{app:pnorm}
The proof of the first claim follows from applying Theorem~\ref{thm:FTRL_regret}. We prove the second claim below.
Here, we consider the regularizer $h(P) = \norm{P - P_1}_{p,P_1}^p$ with $p\ge 2$. 
While this function is not strongly convex, it satisfies the following weaker notion of \emph{$p$-uniform convexity}:
\[
 h(P) \ge h(P') + \iprod{g}{P - P'} + \frac{\alpha}{2} \norm{P - P'}^p_{p,P_1},
\]
with $\alpha = 2$, where $g\in\partial h(P)$. We refer to \citet{BCL94} who attribute this result to \citet{Cla36}.
Following the proof of Lemma~\ref{lem:seminorm-smoothness}, we can show that $h^*$ satisfies the following 
\emph{$q$-uniform smoothness} condition:
\[
 \DDhstar{f}{f'} \le \frac{1}{q\alpha^{q-1}} \norm{f - f'}_{q,P_0,*}^q.
\]
Replacing the inequality used in the last step of the proof of Theorem~\ref{thm:FTRL_regret}, we obtain that FTRL with 
this choice of $h$ satisfies the following regret bound:
\[
 \regret_n(P^*) \le \frac{h\pa{P^*} - h\pa{P_1}}{\eta} + \frac{1}{q}\cdot\pa{\frac{\eta}{2}}^{q-1}\sum_{t=1}^n 
\norm{c_t - g_t}_{q,P_0,*}^q~.
\]
It thus remains to bound the martingale term $-M_{\Pi_n}$, which can be done 
via an application of Lemma~\ref{lem:concentration_heavy_tails} presented in Appendix~\ref{app:concentration}. Indeed, 
applying this result with $X_t  = \EEcc{\loss(\wt{W}_t,Z_t)}{\F_{t-1},S_n}$ implies that for any $\lambda >0$,
we have, with probability at least $1-\delta$,
\[
 -M_{\Pi_n} \le \lambda^{q-1} B^q + \frac{\log \frac 1\delta}{\lambda n}~.
\]
Setting $\lambda = B\pa{\frac{\log \frac 1\delta}{n}}^{1/q}$ concludes the proof.

\section{Technical tools}\label{app:tools}
\subsection{Martingale concentration inequalities}\label{app:concentration}
Here we provide a simple concentration inequality to control the lower tails of sums of nonnegative random variables 
with bounded second moments, used several times in the proofs. 
\begin{lemma}\label{lem:concentration}
Let $\pa{X_t}_{t=1}^n$ be a sequence non-negative random variables 
and for $t\ge 0$, let $\F_t$ denote the $\sigma$-algebra generated by $X_1,\ldots,X_t$.
 Assume that $X_t$ has finite conditional mean $\mu_t = 
\EEcc{X_t}{\F_{t-1}}$ and second moment $\sigma_t^2 = \EEcc{X_t^2}{\F_{t-1}}$. Then, for any $\lambda >0$, the 
following bound holds with probability at least $1-\delta$:
\[
 \sum_{t=1}^n \pa{\mu_t - X_t} \le \frac{\lambda}{2} \sum_{t=1}^n \sigma_t^2 + \frac{\log \frac{1}{\delta}}{\lambda}.
\]
\end{lemma}
\begin{proof}
We use the notation $\EEt{\cdot}= \EE{\cdot |\F_{t-1}}$.
We start by noticing that, for any $\lambda > 0$, we have
\[
 \EEt{e^{- \lambda X_t}} \le \EEt{1 - \lambda X_t + \frac{\lambda^2 X_t^2}{2}} \le e^{- \lambda \mu_t + \lambda^2 
\sigma^2_t/2},
\]
where we have used the inequality $e^{y} \le 1 + y + \frac{y^2}{2}$ that holds for all $y\le 0$.
Using this repeatedly gives
\[
 \EE{e^{\lambda \sum_{t=1}^n \pa{\mu_t - X_t - \lambda \sigma_t^2/2}}} \le 1,
\]
and thus an application of Markov's inequality yields
\[
 \PP{\sum_{t=1}^n \pa{\mu_t - X_t - \lambda \sigma_t^2/2} \ge \varepsilon} = \PP{e^{\lambda \sum_{t=1}^n \pa{\mu_t - 
X_t 
- \lambda \sigma_t^2/2}} \ge e^{\lambda \varepsilon}} \le e^{-\lambda \varepsilon}.
\]
Thus, setting $\varepsilon = \log(1/\delta)/\lambda$ and reordering the terms proves the statement.
\end{proof}

The following simple bound provides an empirical variant of the above bound that holds for bounded random variables.
\begin{lemma}\label{lem:concentration_empirical}
Let $\pa{X_t}_{t=1}^n$ be a sequence random variables supported on $[0,1]$
and for $t\ge 0$, let $\F_t$ denote the $\sigma$-algebra generated by $X_1,\ldots,X_t$.
 Assume that $X_t$ has finite conditional mean $\mu_t = 
\EEcc{X_t}{\F_{t-1}}$. Then, for any $\lambda \in \left[0,\frac 12 \right]$, the 
following bound holds with probability at least $1-\delta$:
\[
 \sum_{t=1}^n \pa{\mu_t - X_t} \le \lambda \sum_{t=1}^n X_t^2 + \frac{\log \frac{1}{\delta}}{\lambda}.
\]
\end{lemma}
\begin{proof}
We use the notation $\EEt{\cdot}= \EE{\cdot |\F_{t-1}}$.
We start by noticing that, for any $\lambda > 0$, we have
\[
 \EEt{e^{- \lambda X_t - \lambda ^2 X_t^2}} \le \EEt{1 - \lambda X_t} \le e^{- \lambda 
\mu_t},
\]
where we have used the inequality $e^{-y-y^2} \le 1 - y$ that holds for all $y \le \frac 12$, which is 
ensured by the conditions on $X_t$ and $\lambda$.
Using this repeatedly gives
\[
 \EE{e^{\lambda \sum_{t=1}^n \pa{\mu_t - X_t - \lambda X_t^2}}} \le 1,
\]
and thus an application of Markov's inequality yields
\[
 \PP{\sum_{t=1}^n \pa{\mu_t - X_t - \lambda X_t^2} \ge \varepsilon} = \PP{e^{\lambda \sum_{t=1}^n \pa{\mu_t - 
X_t 
- \lambda X_t^2}} \ge e^{\lambda \varepsilon}} \le e^{-\lambda \varepsilon}.
\]
Thus, setting $\varepsilon = \log(1/\delta)/\lambda$ and reordering the terms proves the statement.
\end{proof}
Finally, we also supply the following simple extension that applies to heavy-tailed random variables with bounded 
$q$-th moments.
\begin{lemma}\label{lem:concentration_heavy_tails}
Let $\pa{X_t}_{t=1}^n$ be a sequence non-negative random variables with finite conditional mean $\mu_t = 
\EEcc{X_t}{\F_{t-1}}$ and $q$-th moment $B_t = \pa{\EEcc{X_t^q}{\F_{t-1}}}^{1/q}$ for some $q \in (1,2]$. Then, for any 
$\lambda >0$, the following bound holds with probability at least $1-\delta$:
\[
 \sum_{t=1}^n \pa{\mu_t - X_t} \le \lambda^{q-1} \sum_{t=1}^n B_t^q + \frac{\log \frac{1}{\delta}}{\lambda}.
\]
\end{lemma}
\begin{proof}
We start by noticing that, for any $\lambda > 0$, we have
\[
 \EEt{e^{- \lambda X_t}} \le \EEt{1 - \lambda X_t + \lambda^q X_t^q} \le e^{- \lambda \mu_t + \lambda^q 
B^q_t},
\]
where we have used the inequality $e^{y} \le 1 + y + y^q$ that holds for all $y\le 0$.
Using this repeatedly gives
\[
 \EE{e^{\lambda \sum_{t=1}^n \pa{\mu_t - X_t - \lambda^q B^q_t}}} \le 1,
\]
and thus an application of Markov's inequality yields
\[
 \PP{\sum_{t=1}^n \pa{\mu_t - X_t - \lambda^{q-1} B_t^q} \ge \varepsilon} = \PP{e^{\lambda \sum_{t=1}^n 
\pa{\mu_t - X_t 
- \lambda^{q-1} B_t^q}} \ge e^{\lambda \varepsilon}} \le e^{-\lambda \varepsilon}.
\]
Thus, setting $\varepsilon = \log(1/\delta)/\lambda$ and reordering the terms proves the statement.
\end{proof}

\subsection{Strong-convexity / smoothness duality}\label{app:seminorm-smoothness}
\begin{lemma}\label{lem:seminorm-smoothness}
Let $f,f'\in\real^{\Ww}$ and suppose that $h$ is 
$\alpha$-strongly convex with respect to $\norm{\cdot}$. Then, the Legendre--Fenchel conjugate $h^*$ of $h$ satisfies
 \[
  \DDhstar{f}{f'} \le \frac{1}{2\alpha} \norm{f - f'}_{*}^2.
 \]
\end{lemma}
\begin{proof}
 Let $P \in \partial h^*(f)$ and $P' \in \partial h^*(f')$, and also $s_P \in \partial h(P)$ and $s_{P'} \in \partial 
h(P')$. Then, by first-order optimality of $P$ and $P'$, we have
 \begin{align*}
  \iprod{s_P - f}{P - P'} &\le 0\\
  \iprod{s_{P'} - f'}{P' - P} &\le 0.
 \end{align*}
 Summing the two inequalities, we get
 \[
  \iprod{s_{P'} - s_{P}}{P - P'} \le \iprod{P' - P}{f' - f}.
 \]
Now, using the strong convexity of $h$, we get
\begin{align*}
 h(P) \ge h(P') + \iprod{s_{P'}}{P-P'} + \frac{\alpha}{2}\norm{P-P'}^2\\
 h(P') \ge h(P) + \iprod{s_P}{P'-P} + \frac{\alpha}{2}\norm{P-P'}^2.
\end{align*}
Summing these two inequalities then gives
\[
 \alpha \norm{P-P'}^2 \le \iprod{s_P - s_{P'}}{P - P'}.
\]

Combining both inequalities above, we obtain
\[
 \alpha \norm{P-P'}^2 \le \iprod{P - P'}{f-f'} \le \norm{P-P'}\cdot \norm{f-f'}_{*} ,
\]
which implies
\begin{equation}\label{eq:subgrad_norm}
 \norm{P-P'}\le \frac{1}{\alpha} \norm{f-f'}_{*}.
\end{equation}
We continue by expressing $h^*(f) - h^*(f')$ using a form of the mean-value theorem for convex functions (e.g., 
Theorem~2.3.4 in \citet{HL13}). To this end, we define the function $g: [0,1]\ra \mathcal{Q}$ mapping real 
numbers $\lambda \in [0,1]$ to arbitrary elements of the subdifferential of $h^*$ evaluated at $f_\lambda = 
\lambda f + (1-\lambda) f'$, that is, $g(\lambda) \in \partial h^*(f_\lambda)$. Then, we can write
\begin{align*}
 h^*(f) &= h^*(f') + \int_0^1 \iprod{g(\lambda)}{f-f'} \dd \lambda
 \\
 &= h^*(f') + \iprod{P'}{f-f'} + \int_0^1 \iprod{g(\lambda) - P'}{f-f'} \dd \lambda
 \\
 &\le h^*(f') + \iprod{P'}{f-f'} + \frac{1}{\alpha} \int_0^1 \norm{f_\lambda - f'}_{*} \cdot \, \,\norm{f-f'}_{*}\dd 
\lambda
 \qquad\mbox{(by Equation~\ref{eq:subgrad_norm})} 
 \\
 &= h^*(f') + \iprod{P'}{f-f'} + \frac 1\alpha \int_0^1 \lambda  \norm{f - f'}^2_{*} \dd \lambda
 \\
 &= h^*(f') + \iprod{P'}{f-f'} + \frac{1}{2\alpha} \norm{f - f'}^2_{*}.
\end{align*}
The proof is completed by recalling that $P'\in\partial h^*(f')$ and the definition of the Bregman divergence 
$\mathcal{B}_{h^*}$, and reordering the terms.
\end{proof}

\subsection{The proof of Lemma~\ref{lem:smoothing}}\label{app:smoothing}
The proof of the lemma is based on the following successive smoothing argument: we begin by smoothing the 
original function $f$ using the conjugate of the smoothing operator $G_\sigma^*$, then smoothing out the residual $f - 
G_\sigma^* f$ and continue indefinitely. As we show, the residuals decay rapidly at a rate determined by the 
higher-order derivatives of the original function $f$.
To make this argument precise, 
we let $f_0 = f$ and recursively define $f_{j+1} = f_j - G_\sigma^* f_j$, so that we can write
\begin{align*}
 \iprod{P-P'}{f} &= \iprod{P-P'}{G_\sigma^* f} + \iprod{P-P'}{f - G_\sigma^* f} = \iprod{G_\sigma\pa{P-P'}}{f_0} + 
\iprod{P-P'}{f_1}
\\
&= \iprod{G_\sigma \pa{P-P'}}{f_0} + \iprod{P-P'}{G_\sigma^* f_1} + \iprod{P-P'}{f_1 - G_\sigma^* f_1}
\\
&= \iprod{G_\sigma \pa{P-P'}}{f_0} + \iprod{G_\sigma \pa{P-P'}}{f_1} + \iprod{P-P'}{f_2} + \dots
\\
&= \sum_{j=0}^\infty \iprod{G_\sigma \pa{P-P'}}{f_j} \le \norm{P - P'}_\sigma \sum_{j=0}^\infty \norm{f_j}_\infty,
\end{align*}
where the last step follows from  H\"older's inequality. This shows that the dual norm is indeed upper bounded as 
follows:
\[
 \norm{f}_{\sigma,*} = \sup_{\norm{P-P'}_\sigma\le 1} \iprod{P-P'}{f} \le \sum_{j=0}^\infty \norm{f_j}_\infty.
\]

It remains to relate $\norm{f_j}_\infty$ to the derivatives of the original function $f$. 
To this end, let $\xi$ denote a Gaussian vector distributed as $\mathcal{N}(0, \sigma^2 I)$, and
note that for all $j$, we have
\begin{align*}
\infnorm{f_j} &= \sup_w \bigl|f_{j-1}(w) - \EE{f_{j-1}(w+\xi)}\bigr|
\\
&
\le \sup_w \EE{\twonorm{\xi} \cdot \abs{\frac{f_{j-1}(w) - f_{j-1}(w+\xi)}{\twonorm{\xi}}}}
\\
&
\le \EE{\twonorm{\xi}} \sup_w \sup_{v_1\in B_1}\abs{D^1 f_{j-1}(w|v_1)}
\\
&\le \pa{\sigma \sqrt{d}} \sup_w \sup_{v_1\in B_1} \abs{\EE{D^1 f_{j-2}(w|v_1) - D^1 f_{j-2}(w + \xi|v_1)}}
\\
&\le \pa{\sigma \sqrt{d}} \sup_w \sup_{v_1\in B_1} \EE{\twonorm{\xi}\cdot \abs{\frac{D^1 f_{j-2}(w|v_1) - D^1 f_{j-2}(w 
+ 
\xi|v_1)}{\twonorm{\xi}}}}
\\
&\le \pa{\sigma \sqrt{d}} \EE{\twonorm{\xi}} \sup_w \sup_{v_1,v_2\in B_1} \abs{D^2 f_{j-2}(w|v_1,v_2)}
\\
&\le \dots \le \pa{\sigma \sqrt{d}}^j \sup_w \sup_{v_1,v_2,\dots,v_j\in B_1} \abs{D^j f(w|v_1,v_2,\dots,v_j)}\le 
\pa{\sigma \sqrt{d}}^j \beta_j~.
\end{align*}
Here, we have used the bound $\EE{\twonorm{\xi}} \le \sigma \sqrt{d}$ several times.
Putting this together with the previous bound proves the claim.
\qed

\subsection{Wasserstein distance and smoothed relative entropy}\label{app:wasserstein}
This section provides some results supporting the claims made in Section~\ref{sec:wasserstein}. We first give a precise 
definition for the Wasserstein distance between two distributions $P,P'\in\Delta_{\Ww}$. For the sake of concreteness, 
we only give the defintion for the distance metric given by the Euclidean distance on $\real^d$, and refer the reader 
to the book of \citet{Vil03} for a more general treatment. Letting $\Pi(P,P')$ denote the set of joint distributions on 
$\Ww\times\Ww$ with marginals $P$ and $P'$, the squared Wasserstein-2 distance between $P$ and $P'$ is defined as
\[
 \mathbb{W}_2(P,P') = \inf_{\pi \in \Pi(P,P')} \int_{\Ww\times\Ww} \twonorm{w - w'}^2 \dd \pi(w,w').
\]
The following lemma (whose proof is largely based on the proof of Lemma~4 of \citet{NDHR21}) provides a bound on the 
smoothed relative entropy in terms of the squared Wasserstein-2 distance:
\begin{lemma}\label{lem:spreadbound}
Let $W$ and $W'$ be two random variables on $\real^d$ with respective laws $P$ and $P'$. For any $\sigma > 0$, the 
smoothed relative entropy between $P$ and $P'$ is bounded as
\[
\DDsigma{P}{P'} \le \frac{1}{2\sigma^2} \EE{\twonorm{W - W'}^2}.
\]
\end{lemma}
\begin{proof}
Let us consider a fixed coupling $\pi \in \Pi(P,P')$ and observe that the smoothed distributions $G_\sigma P$ and 
$G_\sigma P'$ can be respectively written as 
 \[
  G_\sigma P = \int_{\Ww\times\Ww} \mathcal{N}(w,\sigma^2 I) \dd \pi(w,w') \quad\mbox{and}\quad G_\sigma P' = 
\int_{\Ww\times\Ww} \mathcal{N}(w',\sigma^2 I) \dd \pi(w,w').
 \]
Using this observation, we can write
\begin{align*}
 \DDsigma{P}{P'} &= 
 \DD{\int_{\Ww\times\Ww} \mathcal{N}(w,\sigma^2 I) \dd \pi(w,w')}{\int_{\Ww\times\Ww} \mathcal{N}(w',\sigma^2 
I) \dd \pi(w,w')}
\\
&\le 
\int_{\Ww\times\Ww} \DD{\mathcal{N}(w,\sigma^2I)}{\mathcal{N}(w',\sigma^2I)} \dd \pi(w,w') 
\\
&= \frac{1}{2\sigma^2} \int_{\Ww\times\Ww} \twonorm{W-W'}^2 \dd \pi(w,w'),
\end{align*}
where the second line uses Jensen's inequality and the joint convexity of $\DD{\cdot}{\cdot}$ in its arguments, and the 
last line follows from noticing that $\DD{\mathcal{N}(x,\Sigma)}{\mathcal{N}(y,\Sigma)} = \frac 12 \norm{x 
- y}_{\Sigma^{-1}}^2$ for any $x,y$ and any symmetric positive definite covariance matrix $\Sigma$. The result then 
follows from taking the infimum with respect to $\pi$ on the right-hand side.
\end{proof}

\end{document}